\documentclass[11pt]{article}
\usepackage[letterpaper, margin=1.0in]{geometry}

\usepackage{times}

\usepackage[numbers,sort]{natbib}
\usepackage{microtype}
\usepackage{graphicx}
\usepackage{enumitem,caption,subcaption}
\usepackage{booktabs} %
\usepackage{adjustbox}
\usepackage{bm,mathtools,amsthm,amssymb}

\usepackage{textcomp} %
\usepackage{quoting}
\usepackage{tikz}
\usepackage{titletoc} %
\usepackage{xcolor}
\usepackage{wrapfig}
\usepackage{multirow}
\usepackage{xspace}

\usepackage[hidelinks]{hyperref} %
\usepackage{cleveref} %

\newcommand{\myparagraph}[1]{\paragraph{#1.}\hspace{-0.8em}}  %
\newcommand{\myparagraphsmall}[1]{\vspace*{0.5em}\par\noindent\textbf{{#1}.}}
\newcommand{\printvalues}{topsep=\the\topsep; itemsep=\the\itemsep; parsep=\the\parsep; partopsep=\the\partopsep}

\newcommand{\new}[1]{#1}  %

\usepackage{algorithm}
\usepackage[noend]{algpseudocode}
\algrenewcommand\algorithmicrequire{\textbf{Input:}}
\algrenewcommand\algorithmicensure{\textbf{Output:}}

\newlength{\continueindent}
\usepackage{etoolbox}
\makeatletter
\newcommand*{\ALG@customparshape}{\parshape 2 \leftmargin \linewidth \dimexpr\ALG@tlm+\continueindent\relax \dimexpr\linewidth+\leftmargin-\ALG@tlm-\continueindent\relax}
\apptocmd{\ALG@beginblock}{\ALG@customparshape}{}{\errmessage{failed to patch}}
\makeatother
\algblock{ParFor}{EndParFor}
\algnewcommand\algorithmicpardo{\textbf{in parallel do}}
\algrenewtext{ParFor}[1]{\algorithmicfor\ #1\ \algorithmicpardo}
\algrenewtext{EndParFor}{\textbf{end for}}
\algtext*{EndParFor}
\algdef{SE}[SUBALG]{Indent}{EndIndent}{}{\algorithmicend\ }%
\algtext*{Indent}
\algtext*{EndIndent}

\definecolor{Code}{rgb}{0,0,0} 
\definecolor{Decorators}{rgb}{0.5,0.5,0.5} 
\definecolor{Numbers}{rgb}{0.5,0,0} 
\definecolor{MatchingBrackets}{rgb}{0.25,0.5,0.5} 
\definecolor{Keywords}{rgb}{0,0,1} 
\definecolor{self}{rgb}{0,0,0} 
\definecolor{Strings}{rgb}{0,0.63,0} 
\definecolor{Comments}{rgb}{0.63,0,0} 
\definecolor{Backquotes}{rgb}{0,0,0} 
\definecolor{Classname}{rgb}{0,0,0} 
\definecolor{FunctionName}{rgb}{0,0,0} 
\definecolor{Operators}{rgb}{0,0,0} 
\definecolor{Background}{rgb}{0.99,0.99,0.99} 
 
\usepackage{setspace}
\usepackage{listings}
\usepackage{color}
\definecolor{puorange}{rgb}{0.80,0.20,0}
\definecolor{bluegray}{rgb}{0.04,0,0.7}
\definecolor{greengray}{rgb}{0.05,0.50,0.15}
\definecolor{darkbrown}{rgb}{0.40,0.2,0.05}
\definecolor{darkcyan}{rgb}{0,0.4,1}
\definecolor{black}{rgb}{0,0,0}
\definecolor{grey}{rgb}{0.93,0.93,0.93}

\newtheorem{theorem}{Theorem}
\newtheorem{lemma}[theorem]{Lemma}
 \newtheorem{claim}[theorem]{Claim}

\newtheorem{proposition}[theorem]{Proposition}
\newtheorem{corollary}[theorem]{Corollary}
\newtheorem{remark}[theorem]{Remark}
\newtheorem{definition}[theorem]{Definition}
\newtheorem{assumption}[theorem]{Assumption}

\crefname{section}{Section}{Sections}
\crefname{appendix}{Appendix}{Appendices}

\crefname{theorem}{Theorem}{Theorems}
\crefname{lemma}{Lemma}{Lemmas}
\crefname{corollary}{Corollary}{Corollaries}
\crefname{proposition}{Proposition}{Propositions}
\crefname{definition}{Definition}{Definitions}

\Crefname{algorithm}{Algorithm}{Algorithms}
\crefname{figure}{Figure}{Figures}
\crefname{table}{Table}{Tables} %

\newcommand{\No}{-}
\newcommand{\Yes}{Yes}

\newcommand{\fedavg}{FedAvg\xspace}
\newcommand{\rbfedavg}{RFA\xspace}

\newcommand{\niidterm}{\Omega}

\newcommand \reals {\mathbb{R}}

\newcommand \inv {^{-1}} %
\newcommand \T {^{\top}}	%

\newcommand \bigO {\mathcal{O}}

\newcommand \Ccal {\mathcal{C}}

\newcommand \prob {\mathbb{P}}
\newcommand \expect {\mathbb{E}}

\newcommand \pow [1]{^{(#1)}}
\DeclarePairedDelimiterX{\inp}[2]{\langle}{\rangle}{#1, #2} %
\DeclarePairedDelimiterX{\norm}[1]{\Vert}{\Vert}{#1} %
\DeclarePairedDelimiterX{\normsq}[1]{\Vert}{\Vert^2}{#1} %

\newcommand \norma [2]{\Vert #2 \Vert_{#1}}

\newcommand \eps \epsilon

\newcommand \argmin {\operatorname*{arg\,min}} %
\newcommand \argmax {\operatorname*{arg\,max}} %
\newcommand \conv {\operatorname*{conv}} %

\newcommand \grad {\nabla}

\DeclareMathOperator\diam{diam}

 \usepackage{thmtools}
 \declaretheoremstyle[
notefont=\bfseries, notebraces={}{},
bodyfont=\normalfont\itshape,
headformat=\NAME \NOTE
]{nopar}

\title{Robust Aggregation for Federated Learning}

\author{
Krishna~Pillutla$^1$ $\qquad\qquad\qquad$
        Sham M. Kakade$^2$ $\qquad\qquad\qquad$
        Zaid Harchaoui$^1$ \\
    {\small
    $^1$ University of Washington $\qquad$
    $^2$ Harvard University
    }
}

\date{\vspace{-2em}}

\begin{document}
\maketitle

\begin{abstract}

Federated learning is the centralized training of statistical models from decentralized data on mobile devices while preserving the privacy of each device. We present a robust aggregation approach to make federated learning robust to settings when a fraction of the devices may be sending corrupted updates to the server. The approach relies on a robust aggregation oracle based on the geometric median, which returns a robust aggregate using a constant number of iterations of a regular non-robust averaging oracle. The robust aggregation oracle is privacy-preserving, similar to the non-robust secure average oracle it builds upon.
We establish its convergence for least squares estimation of additive models. We provide experimental results with linear models and deep networks for three tasks in computer vision and natural language processing. The robust aggregation approach is agnostic to the level of corruption; it outperforms the classical aggregation approach in terms of robustness when the level of corruption is high, while being competitive in the regime of low corruption.
Two variants, a faster one with one-step robust aggregation and another one with on-device personalization, round off the paper. \end{abstract}

\section{Introduction} \label{sec:intro}
Federated learning is a key paradigm for machine learning 
and analytics on mobile, wearable and edge devices~\cite{mcmahan2017communication,kairouz2019flsurvey}
over wireless networks of 5G and beyond as well as edge networks and the internet of things. 
The paradigm has found widespread applications ranging from mobile apps deployed on millions of devices~\citep{yang2018applied,ammad2019federated}, to sensitive healthcare applications~\citep{pantelopoulos2009survey,huang2019patient}. 

In federated learning, a number of devices with privacy-sensitive data collaboratively optimize a machine learning model under the orchestration of a central server, while keeping the data fully decentralized and private.
Recent work has looked beyond supervised learning to domains such as data analytics but also semi-, self- and un-supervised learning, transfer learning, meta learning, and reinforcement learning~\cite{kairouz2019flsurvey,ren2019federated,lin2020collaborative,zhuang2021collaborative}.

We study a question relevant in all these areas: robustness to corrupted updates. Federated learning relies on aggregation of updates contributed by participating devices, where the aggregation is privacy-preserving. 
Sensitivity to corrupted updates, caused either by adversaries intending to attack the system or due to failures in low-cost hardware, is a vulnerability of the usual approach. The standard arithmetic mean aggregation in federated learning is not robust to corruptions, in the sense that even a single corrupted update in a round is sufficient to degrade the global model for all devices. In one dimension, the median is an attractive aggregate for its robustness to outliers. We adopt this approach to federated learning by considering a classical multidimensional generalization of the median, known variously as the geometric or spatial or $L_1$ median~\cite{maronna:etal:2006}. 

Our robust approach preserves the privacy of the device updates by iteratively invoking the secure multi-party computation primitives used in typical non-robust federated learning~\cite{bonawitz2017practical,bell2020secure}. 
A device's updates are information theoretically protected in that they are computationally indistinguishable from random noise and the sensitivity of the final aggregate to the contribution of each device is bounded. Our approach is scalable, since the underlying secure aggregation algorithms are implemented in production systems across millions of mobile users across the planet~\cite{bonawitz2019towards}.
The approach is communication-efficient, requiring a modest 1-3$\times$ the communication cost of the non-robust setting to compute the non-linear aggregate in a privacy-preserving manner. 

\myparagraph{Contributions}
The main take-away message of this work is:
\begin{quoting}[indentfirst=false,vskip=0.3em,leftmargin=1em,rightmargin=1em]
	{\em Federated learning can be made robust to corrupted updates 
	by replacing the weighted arithmetic mean
	aggregation with an approximate geometric median
	at {1-3 times} the communication cost.}
\end{quoting}
To this end, we make the following concrete contributions.

\begin{enumerate}[label=(\alph*),nolistsep,leftmargin=\widthof{ (a) }]
\item{\textit{Robust Aggregation:}}
We design a novel robust aggregation oracle based  
on the classical geometric median. 
We analyze the convergence of the resulting federated learning algorithm, \rbfedavg,
for least-squares estimation
and show that the proposed method is robust to
update corruption in up to half the devices
in federated learning with bounded heterogeneity.
We also describe an extension of the framework to handle 
arbitrary heterogeneity via personalization.

\item{\textit{Algorithmic Implementation:}}
We show how to implement this robust aggregation oracle
in a practical and privacy-preserving manner.
This relies on an alternating minimization algorithm 
which empirically exhibits rapid convergence.
This algorithm can be interpreted as a numerically 
stable version of the classical algorithm of Weiszfeld~\cite{weiszfeld1937point},
thus shedding new light on it. 

\item{\textit{Numerical Simulations:}}
We demonstrate the effectiveness of our framework for
data corruption and parameter update corruption,
on federated learning tasks from computer vision and natural language processing, with linear models as well as convolutional and recurrent neural networks. In particular, our results show that 
the proposed \rbfedavg algorithm 
(i) outperforms the standard \fedavg~\citep{mcmahan2017communication}, in high corruption 
and 
(ii) nearly matches the performance of the \fedavg in low corruption, both at
{\em 1-3 times the communication cost}.
Moreover, the proposed algorithm is agnostic to the 
actual level of corruption in the problem instance.

\end{enumerate}
\noindent
We open source an implementation of the proposed approach in TensorFlow Federated~\cite{rfa_tff};
cf. Appendix~\ref{sec:a:software} for a template implementation.
The Python code and 
scripts used to reproduce experimental results are publicly available online~\cite{rfa_repo}. 

\myparagraph{Overview}
\Cref{sec:related} describes related work, and
\Cref{sec:setup} describes the problem formulation and tradeoffs of robustness.
\Cref{sec:aggr} proposes a robust aggregation oracle 
and presents a convergence analysis of the resulting robust federated learning algorithm.
Finally, \Cref{sec:expt} gives comprehensive numerical simulations 
demonstrating the robustness of the proposed federated learning algorithm compared to standard baselines.

\section{Related Work} \label{sec:related}
\textit{Federated Learning} was introduced in~\citep{mcmahan2017communication}
as a distributed optimization approach to handle on-device machine learning, 
with secure multi-party averaging algorithms given in~\citep{bonawitz2017practical,balle2020hypothesis}.
Extensions were proposed in~\citep{smith2017federated,mohri2019agnostic,sahu2018convergence,karimireddy2020scaffold,dinh2020personalized,fallah2020personalized,laguel2021superquantile,avdiukin2021federated,reddi2021adaptive}; see also the recent surveys~\citep{li2020federated,kairouz2019flsurvey}.
We address robustness to corrupted updates, which is broadly applicable in these settings.

\textit{Distributed optimization} has a long history~\citep{bertsekas1989parallel}.
Recent work includes primal-dual frameworks~\citep{smith2018cocoa,ma2017distributed} 
and variants suited to
decentralized~\citep{he2018cola}, and
asynchronous~\citep{leblond2018improved} settings.
From the lens of \textit{learning in networks}~\citep{sayed2014adaptation},
federated learning comprises a star network where 
agents (i.e., devices) with private data are connected to a server with no data,
which orchestrates the cooperative learning.
Further, for privacy, model updates from individual agents
cannot be shared directly, but must be aggregated securely.

\textit{Robust estimation} was pioneered by Huber~\cite{huber1964,huber2011robust}.
Robust median-of-means were introduced in
\citep{nemirovsky1983problem}, with follow ups
in~\citep{minsker2015geometric,hsu2016loss,lugosi2016risk,lecue2017robust,lugosi2017regularization}.
Robust mean estimation, in particular, received much attention~\citep{diakonikolas2016robust,minsker2018uniform,cheng2019high}.
Robust estimation in networks was considered in~\cite{al2017robust,yu2019robust,chen2019resilient}.
These works consider the statistics of robust estimation in the i.i.d. case, 
while we focus on distributed optimization with privacy preservation.

\textit{Byzantine robustness}, resilience to
arbitrary behavior of some devices~\citep{lamport1982byzantine},
was studied
in distributed optimization with gradient aggregation~\citep{blanchard2017machine,chen2017distributed,chen2018draco,yin2018byzantine,alistarh2018byzantine,cao2019distributed}.
Byzantine robustness of federated learning is a priori not possible without additional assumptions because the secure multi-party computation protocols require faithful  participation of the devices.
Thus, we consider a more nuanced and less adversarial corruption model
where devices participate faithfully in the aggregation loop;
see \Cref{sec:setup} for practical examples.
Further, it is unclear how to securely implement the nonlinear aggregation algorithms of these works.
Lastly, the use of, e.g., secure enclaves~\citep{subramanyan2017formal}
in conjunction with our approach could guarantee Byzantine robustness in federated learning.
We aggregate {\em model parameters} in a robust manner,  which is more suited to the federated setting. We note that \cite{li2019rsa} also aggregate model parameters rather than gradients by framing the problem in terms of consensus optimization. However, their algorithm requires devices to be always available and participate in multiple rounds, which is not practical in the federated setting~\cite{kairouz2019flsurvey}.

Weiszfeld's algorithm~\citep{weiszfeld1937point} to \textit{compute the geometric median}, 
has received much attention
\citep{kuhn1973note,vardi2001modified,beck15weiszfeld}. The Weiszfeld algorithm is also known to exhibit asymptotic linear convergence~\cite{katz1974local}.
However, unlike these variants, ours is numerically stable.
A theoretical proposal of a near-linear time algorithm for the
geometric median was recently explored in~\citep{cohen2016linear}. 
Frameworks to guarantee \textit{privacy} of user data include
differential privacy~\citep{dwork2006calibrating,kairouz2021distributed} 
and homomorphic encryption~\citep{gentry2010computing}.
These directions are orthogonal to ours, and could be 
used in conjunction.
See~\citep{bonawitz2017practical,li2020federated,kairouz2019flsurvey} for a broader discussion.

\section{Problem Setup: Federated Learning with Corruptions} \label{sec:setup}
We begin this section by recalling the setup of federated learning (without corruption) and the standard \fedavg algorithm~\cite{mcmahan2017communication} in \Cref{sec:setup:fl}. 
We then formally setup our corruption model and discuss the trade-offs introduced by requiring robustness to corrupted updates in \Cref{sec:setup:tradeoffs}.

\subsection{Federated Learning Setup and Review} \label{sec:setup:fl}
Federated learning consists of $n$
 client devices which collaboratively 
train a machine learning model under the orchestration of a central server or a fusion center~\cite{mcmahan2017communication,kairouz2019flsurvey}. The data is local to the client devices while the job of the server is to orchestrate the training. 

We consider a typical federated learning setting where each device $i$
has a distribution $D_i$ over some data space such that the data on the client is sampled i.i.d. from $D_i$.
Let the vector $w \in \reals^d$ denote the parameters of a 
(supervised) learning model and let $f(w; z)$ denote the loss of model $w$ on input-output pair $z$, such as the mean-squared-error loss. Then, the objective function of device $i$ is $F_i(w) = \expect_{z \sim D_i} \left[ f(w; z)\right]$. 

Federated learning aims to find a model $w^\star$ that minimizes the average objective across all the devices, 
\begin{align} \label{eq:fl:main}
	 \min_{w \in \reals^d} \left[ F(w) := 
	 \sum_{i=1}^n \alpha_i \, F_i(w) \right] \,,
\end{align}
where device $i$ is weighted by $\alpha_i > 0$.
In practice, the weight $\alpha_i$ is chosen proportional to 
the amount of data on device $i$. 
For instance, in an empirical risk minimization setting, each $D_i$ is the uniform distribution over a finite set $\{z_{i, 1}, \cdots, z_{i, N_i}\}$ of size $N_i$. It is common practice to choose $\alpha_i = N_i/N$ where $N=\sum_{i=1}^n N_i$ so that the objective $F(w) = (1/N) \sum_{i=1}^n \sum_{j=1}^{N_i} f(w; z_{i, j})$ is simply the unweighted average over all samples from all $n$ devices. 

\myparagraph{Federated Learning Algorithms}
Typical federated learning algorithms run in synchronized rounds of communication between 
the server and the devices with some local computation on the devices based on their local data, and aggregation of these updates to update the server model.
The de facto standard training algorithm is  \fedavg~\cite{mcmahan2017communication}, which runs as follows.
\begin{enumerate}[label=(\alph*),nolistsep,leftmargin=\widthof{ (a) }]
    \item 
    The server samples a set $S_t$ of $m$ clients from $[n]$ and broadcasts the current model $w\pow{t}$ to these clients. 
    \item 
    Staring from $w_{i, 0}\pow{t} = w\pow{t}$, each client $i \in S_t$ makes $\tau$ local gradient or stochastic gradient descent steps for $k=0, \cdots, \tau-1$ with a learning rate $\gamma$:
    \begin{align}  \label{eq:rfa:fedavg:local-update}
        w_{i, k+1}\pow{t} = w_{i, k}\pow{t} - \gamma \grad F_i(w_{i, k}\pow{t}) \,.
    \end{align}
    \item Each device $i \in S_t$ sends to the server a vector $w_i\pow{t+1}$ which is simply the final iterate, i.e., $w_i\pow{t+1}=w_{i,\tau}\pow{t}$. The server updates its global model using the weighted average
    \begin{align} \label{eq:rfa:fedavg:averaging}
        w\pow{t+1} = \frac{\sum_{i \in S_t} \alpha_i  w_{i}\pow{t+1}}{\sum_{i \in S_t} \alpha_i} \,.
    \end{align}
\end{enumerate}

The federated learning algorithm, and in particular, the choice of aggregation, impacts the following three factors~\cite{kairouz2019flsurvey,li2020federated,gafni2021federated}:
communication efficiency, privacy, and robustness.

\myparagraph{Communication Efficiency}
Besides the computation cost, the communication cost is an important parameter in distributed optimization. While communication is relatively fast in the datacenter, that is not the case of federated learning.  
The repeated exchange of massive models between the server and client devices over resource-limited wireless networks makes communication over the network more of a bottleneck in federated learning than local computation on the devices. 
Therefore, training algorithms should be able to trade-off more local computation for lower communication, similar to step (b) of \fedavg above. While the exact benefits (or lack thereof) of local steps is an active area of research, local steps have been found empirically to reduce the amount of communication required for a moderately accurate solution~\cite{mcmahan2017communication,wang2021field}.

Accordingly, we set aside the local computation cost for a first order approximation, and compare algorithms in terms of their total communication cost~\cite{kairouz2019flsurvey}. 
Since typical federated learning algorithms proceed in synchronized rounds of communication, we measure the complexity of the algorithms in terms of the number of communication rounds.

\myparagraph{Privacy}
While the privacy-sensitive data $z \sim D_i$ is kept local to the device, the model updates $w_i\pow{t+1}$ might also leak privacy. To add a further layer of privacy protection, the server is not allowed to inspect individual updates $w_i\pow{t+1}$ in the aggregation step (c); it can only access the aggregate $w\pow{t+1}$.

We make this precise through the notion of a \textit{secure average oracle}. 
	Given $m$ devices with each device $i$ containing $w_i \in \reals^d$ and a scalar $\beta_i > 0$, 
	a secure average oracle computes the average $\sum_{i=1}^m \beta_i w_i / \sum_{i=1}^m \beta_i$ 
	at a total communication of $\bigO(md + m \log m)$ bits 
	such that no $w_i$ or $\beta_i$ are revealed
	to either the server or any other device.
	
In practice, a secure average oracle is implemented using 
cryptographic protocols based on secure multi-party computation~\cite{bonawitz2017practical,bell2020secure}. These require a communication overhead of $O(m\log m)$ in addition to $O(md)$ cost of sending the $m$ vectors. First, the vector $\beta_i w_i$ is dimension-wise discretized on the ring $\mathbb{Z}^d_M$ of integers modulo $M$ in $d$-dimensions. Then, a noisy version $\tilde w_i$ is sent to the server, where the noise is designed to satisfy:
\begin{itemize}[nolistsep,leftmargin=\widthof{ (a)}]
\item correctness up to discretization, by ensuring  $\sum_{i=1}^m \tilde w_i \mod M = \sum_{i=1}^m \beta_i w_i \mod M$ with probability 1, and,
\item privacy preservation from honest-but-curious devices and server in the information theoretic sense, by ensuring that $\tilde w_i$ is computationally indistinguishable from $\zeta_i \sim \mathrm{Uniform}(\mathbb{Z}^d_M)$, irrespective of $w_i$ and $\beta_i$.
\end{itemize}
As a result, we get the correct average (up to discretization) while not revealing any further information about a $w_i$ or $\beta_i$ to the server or other devices, beyond what can be inferred from the average.
Hence, no further information about the underlying data distribution $D_i$ is revealed either.
In this work, we assume for simplicity that the secure average oracle returns the exact update, i.e., we ignore the effects of discretization on the integer ring and modular wraparound. This assumption is reasonable for a large enough value of $M$.

\myparagraph{Robustness}
We would like a federated learning algorithm to be robust to corrupted updates contributed by malicious devices or hardware/software failures. 
\fedavg uses an arithmetic mean to aggregate the device updates in \eqref{eq:rfa:fedavg:averaging}, which is known to not be robust~\cite{huber1964}. 
This can be made precise by the notion of a breakdown point~\cite{donoho1983notion}, which is the smallest fraction of the points which need to be changed to cause the aggregate to take on arbitrary values. The breakdown point of the mean is 0, since only one point needs to changed to arbitrarily change the aggregate~\cite{maronna:etal:2006}. 
This means in federated learning that a single corrupted update, either due to an adversarial attack or a failure, can arbitrarily change the resulting aggregate in each round. We will give examples of adversarial corruptions in \Cref{sec:setup:tradeoffs}.

In the rest of this work, we aim to address the lack of robustness of \fedavg. A popular robust aggregation of scalars is the median rather than the mean. We investigate a multidimensional analogue of the median, while respecting the other two factors: communication efficiency and privacy. While the non-robust mean aggregation can be computed with secure multi-party computation via the secure average oracle, it is unclear if a robust aggregate can also satisfy this requirement.
We discuss this as well as other tradeoffs involving robustness in the next section.

\begin{table*}[t]
\captionsetup{singlelinecheck=off}
\caption{
\small{
Examples corruptions and
capability of an adversary they require, as measured
along the following axes:
	\textbf{Data write}, where a device $i \in \mathcal{C}$ 
		can replace its local distribution $D_i$ by any arbitrary distribution $\tilde D_i$; 
	\textbf{Model read}, where a device $i \in \mathcal{C}$ 
		can read the server model $w\pow{t}$ and replace its local distribution $D_i$  by an adaptive distribution $\tilde D_i\pow{t}$ depending on $w\pow{t}$;
	\textbf{Model write}, where a device $i \in \mathcal{C}$ 
		can return an arbitrary vector to the server for aggregation as in \eqref{eq:setup:corruption_model}, and,
	\textbf{Aggregation}, where a device $i \in \mathcal{C}$ 
		can behave arbitrarily during the computation of an iterative secure aggregate. 
The last column indicates whether the proposed RFA algorithm is robust to each type of corruption.
}
}
\label{table:corrupt:example}
\begin{center}
\begin{adjustbox}{max width=\linewidth}
\begin{tabular}{lccccc}
\toprule
\, \textbf{Corruption Type} & 
\textbf{Data write} & \textbf{Model read} & \textbf{Model write} & 
\textbf{Aggregation} & \textbf{RFA applicable?} \\
\midrule

\begin{tabular}{l} Non-adversarial \end{tabular} & 
    \No & \No & \No & \No & \checkmark \\

\begin{tabular}{l} Static data poisoning \end{tabular} & 
    \Yes & \No & \No & \No & \checkmark \\

\begin{tabular}{l} Adaptive data poisoning \end{tabular} & 
    \Yes & \Yes & \No & \No & \checkmark \\

\begin{tabular}{l} Update poisoning \end{tabular} & 
    \Yes & \Yes & \Yes & \No & \checkmark \\

\begin{tabular}{l} Byzantine \end{tabular} & 
    \Yes & \Yes & \Yes & \Yes & N/A \\

\bottomrule
\end{tabular}
\end{adjustbox}
\end{center}
\vskip -0.1in
\end{table*}

\subsection{Corruption Model and Trade-offs of Robustness} \label{sec:setup:tradeoffs}
We start with the corruption model used in this work. 
We allow a subset $\Ccal \subset [n]$ of \emph{corrupted devices} to, 
unbeknownst to the server, send arbitrary vectors $w_i\pow{t+1} \in \reals^d$ rather than the updated model $w_{i, \tau}\pow{t}$ from local data as expected by the server. Formally, we have, 
\begin{align} \label{eq:setup:corruption_model}
    w_i\pow{t+1} = 
    \begin{cases}
        w_{i, \tau}\pow{t}\,, & \text{ if } i \notin \Ccal,\\
        H_i\left(w\pow{t}, \{(w_{j, \tau}\pow{t}, D_{j})\}_{j \in S_t}\right)\, &\text{ if } i \in \Ccal,
    \end{cases}
\end{align}
where $H_i$ is an arbitrary $\reals^d$-valued function which is allowed to depend on the global model $w\pow{t}$, the uncorrupted updates $w_{j, \tau}\pow{t}$ as well as the data distributions $D_{j}$ of each device $j \in S_t$. 

This encompasses situations where the corrupted devices are individually or collectively trying to ``attack'' the global model, that is, reduce its predictive power over uncorrupted data. 
We define the \emph{corruption level} $\rho$ as the total fraction of the weight of the corrupted devices:
\begin{align}
    \rho = \frac{\sum_{i \in \Ccal} \alpha_i}{\sum_{i=1}^n \alpha_i} \,.
\end{align}

Since the corrupted devices can only harm the global model through the updates they contribute in the aggregation step, we aim to robustify the aggregation in federated learning. 
However, it turns out that robustness is not directly compatible with 
the two other desiderata of federated learning, namely communication efficiency and privacy.

\myparagraph{The Tension Between Robustness, Communication and Privacy}
We first argue that any federated learning algorithm can only have two out of the three of robustness, communication and privacy under the existing techniques of secure multi-party computation. 
The standard approach of \fedavg is communication-efficient and privacy-preserving but not robust, as we discussed earlier. 
In fact, any aggregation scheme $A(w_1, \cdots, w_m)$ which is a linear function of $w_1, \cdots, w_m$ is similarly non-robust.
Therefore, any robust aggregate $A$ must be a non-linear function of the vectors it aggregates. 

The approach of sending the updates to the server at a communication of $O(md)$ and utilizing one of the many robust aggregates studied in the literature~\cite[e.g.][]{chen2017distributed,yin2018byzantine,alistarh2018byzantine} has robustness and communication efficiency but not privacy. If we try to make it privacy-preserving, however, we lose communication efficiency. Indeed, the secure multi-party computation primitives based on secret sharing, upon which privacy-preservation is built, are communication efficient only for linear functions of the inputs~\cite{evans2018pragmatic}. The additional $O(m \log m)$ overhead of secure averaging for linear functions becomes $\Omega(m d \log m)$ for general non-linear functions required for robustness; this makes it impractical for large-scale systems~\cite{bonawitz2017practical}. 
Therefore, one cannot have both communication efficiency and privacy preservation along with robustness. 

In this work, we strike a compromise between robustness, communication and privacy. We will approximate a non-linear robust aggregate as an \emph{iterative secure aggregate}, i.e., as a sequence of weighted averages, computed with a secure average oracle with weights being adaptively updated. 
\begin{definition} \label{def:setting:secure_agg:iterative}
    A function $A:(\reals^{d})^m \to \reals^d$ is said to be an iterative secure aggregate of $w_1, \cdots, w_m$ with $R$ communication rounds 
    and initial iterate $v\pow{0}$ if 
    for $r=0, \cdots, R-1$, there exist weights $\beta_1\pow{r}, \cdots, \beta_m\pow{r}$ such that
    \begin{enumerate}[label=(\roman*),nolistsep,leftmargin=\widthof{ (iii) }]
        \item  $\beta_i\pow{r}$ depends only on $v\pow{r}$ and $w_i$,
        \item $v\pow{r+1} = \sum_{i=1}^m \beta_i\pow{r} w_i / \sum_{i=1}^m \beta_i\pow{r}$, and, 
        \item $A(w_1, \cdots w_m) = v\pow{R}$.
    \end{enumerate}
    Further, the iterative secure aggregate is said to be $s$-privacy preserving for some $s\in(0, 1)$ if 
    \begin{enumerate}[resume,label=(\roman*),nolistsep,leftmargin=\widthof{ (a) }]
    \item $\beta_i\pow{r} / \sum_{j=1}^m \beta_{j}\pow{r} \le s$ for all $i \in [m]$ and $r \in [R]$. 
    \end{enumerate}
\end{definition}
If we have an iterative secure aggregate with $R$ communication rounds which is also robust, we gain robustness at a $R$-fold increase in communication cost.
Condition (iv) ensures privacy preservation because it reveals only weighted averages with weights at most $s$, so a user's update is only available after being mixed with those from a large cohort of devices.

\myparagraph{The Tension Between Robustness and Heterogeneity}
Heterogeneity is a key property of federated learning. The distribution $D_i$ of device $i$ can be quite different from the distribution $D_j$ of some other device $j$, reflecting the heterogeneous data generated by a diverse set of users. 

To analyze the effect of heterogeneity on robustness, consider the simplified scenario of robust mean estimation in Huber's contamination model~\cite{huber1964}. Here, we wish to estimate the mean $\mu \in \reals^d$ given samples $w_1, \cdots, w_m \sim (1-\rho)\mathcal{N}(\mu, \sigma^2 I) + \rho Q$, where $Q$ denotes some outlier distribution that $\rho$-fraction of the points (designated as outliers) are drawn from. 
Any aggregate $\bar w$ must satisfy the lower bound 
$\normsq{\bar w - \mu} \ge \Omega\big(\sigma^2 \max\{\rho^2, d/m\}\big)$ with constant probability~\cite[Theorem 2.2]{chen2018robust}.
In the federated learning setting, more heterogeneity corresponds to a greater variance $\sigma^2$ among the inlier points, implying a larger error in mean estimation. This suggests a tension between robustness and heterogeneity, where increasing heterogeneity makes robust mean estimation harder in terms of $\ell_2$ error.

In this work, we strike a compromise between robustness and heterogeneity by considering a family $\mathcal{D}$ 
of allowed data distributions such that any device $i$ with $D_i \notin \mathcal{D}$ will be regarded as a corrupted device, i.e., $i \in \Ccal$. We will be able to guarantee convergence up to 
the degree of heterogeneity in $\mathcal{D}$; we call this $\mathrm{width}(\mathcal{D})$ and make it precise in Section~\ref{sec:aggr}. In the i.i.d. case, $\mathcal{D}$ is a singleton and $\mathrm{width}(\mathcal{D}) = 0$.

\myparagraph{Examples}
Next, we consider some examples of update corruption ---
see~\citep{kairouz2019flsurvey} for a comprehensive treatment.
Corrupted updates could be non-adversarial in nature, such as sensor malfunctions or hardware bugs in unreliable and heterogeneous
devices (e.g., mobile phones) which are outside the control of the orchestrating server.
On the other hand, we could also have adversarial corruptions of the following types:
\begin{enumerate}[label=(\alph*),nolistsep,leftmargin=\widthof{ (a) }]
\item \textit{Static data poisoning:}
The corrupted devices $\mathcal{C}$ are allowed to 
modify their training data prior to the start of 
the training, and the data is fixed thereafter. 
Formally, the objective function of device $i \in \mathcal{C}$ is now
$\tilde F_i(w) = \expect_{z \sim \tilde D_i}\left[f(w;z)\right]$
where $\tilde D_i$ has been modified from the original ${D}_i$.
These devices then participate in the local updates \eqref{eq:rfa:fedavg:local-update} with $\grad \tilde F_i$ rather than $\grad F_i$. 
We consider device $i$ to contribute corrupted updates only if
$\tilde D_i \notin \mathcal{D}$ (for instance, $\mathcal{D}$ is the set of natural RGB images).

\item{\textit{Adaptive data poisoning:}}
The corrupted devices $\mathcal{C}$
are allowed to modify their training data in each round of training
depending on the current model $w\pow{t}$. Concretely, the objective function of device $i \in \mathcal{C}$ in round $t$ is 
$\tilde F_i\pow{t}(w) = \expect_{z \sim \tilde D_i\pow{t}}\left[f(w;z)\right]$
where $\tilde D_i\pow{t}$ has been modified from the original ${D}_i$ using knowledge of $w\pow{t}$.
As previously, these devices then participate in the local updates \eqref{eq:rfa:fedavg:local-update} with $\grad \tilde F_i\pow{t}$ rather than $\grad F_i$ in round $t$. 

\item{\textit{Update Poisoning:}}
The corrupted devices can send an arbitrary vector to the server for aggregation, as described by \eqref{eq:setup:corruption_model} in its full generality. This setting subsumes all previous examples as special cases.

\end{enumerate}

The corruption model in \eqref{eq:setup:corruption_model} precludes 
the {\em Byzantine setting}~\cite[e.g.,][Sec. 5.1]{kairouz2019flsurvey}, which refers to the worst-case model where 
a corrupted client device $i \in \mathcal{C}$ can behave arbitrarily, such as for instance, changing the weights $\beta_i\pow{r}$ or the vector $w_i$ between each of the rounds of the iterative secure aggregate, as defined in Definition~\ref{def:setting:secure_agg:iterative}.
It is provably impossible to design 
a Byzantine-robust iterative secure aggregate in this sense. The examples listed above highlight the importance of robustness to the corruption model under consideration.

Table~\ref{table:corrupt:example} compares 
the various corruptions in terms of the capability of an adversary required 
to induce the corruption.

\section{Robust Aggregation and the \rbfedavg{} Algorithm} \label{sec:aggr}
\begin{algorithm}[t]
	\caption{The \rbfedavg Algorithm}
	\label{algo:rfa:main}
\begin{algorithmic}[1]
		\Require Initial iterate $w\pow{0}$,
		    number of communication rounds $T$, 
		    number of clients per round $m$, 
		    number of local updates $\tau$,
		    local step size $\gamma$, approximation threshold $\eps$
	    \For{$t=0, 1, \cdots, T-1$}
	        \State Sample $m$ clients from $[n]$ without replacement in $S_t$ \label{line:rfa:sample}
	        \For{each selected client $i \in S_t$ in parallel}
	            \State Initialize $w_{i, 0}\pow{t} = w\pow{t}$ \label{line:rfa:local-update-start}
	    	    \For{$k=0, \cdots, \tau-1$} \label{line:rfa:local}
	    	        \State Sample data $z_{i, k}\pow{t} \sim D_i$
	    	        \State Update $w_{i, k+1}\pow{t} = w_{i, k}\pow{t} - \gamma \grad f(w_{i, k}\pow{t}; z_{i, k}\pow{t})$ 
	    	        \State Set $w_i\pow{t+1} = w_{i, \tau}\pow{t}$ 
	    	            \label{line:algo:rfa:local-update-final}
	    	    \EndFor
	    	\EndFor
	    	\State $w\pow{t+1} = \mathrm{GM}\big( (w_i\pow{t+1})_{i\in S_t}, (\alpha_i)_{i \in S_t}, \eps\big)$ (Algo.~\ref{algo:rfa:weiszfeld})
	         \label{line:rfa:aggregation}
	    \EndFor
	    \State \Return $w_T$
\end{algorithmic}
\end{algorithm}

In this section, we design a robust aggregation oracle
and analyze the convergence of the resulting 
federated algorithm.

\myparagraph{Robust Aggregation with the Geometric Median}
The geometric median (GM) of $w_1,\cdots, w_m \in \reals^d$ 
with weights $\alpha_1, \cdots, \alpha_m > 0$ 
is the minimizer of
\begin{align} \label{eq:geom_med:nonsmooth}
	g(v) := \sum_{i=1}^m \alpha_i \norm{v-w_i} \,,
\end{align}
where $\norm{\cdot}=\norma{2}{{\cdot}}$ is the Euclidean norm.
As a robust aggregation oracle, we use an
$\eps$-approximate minimizer $\widehat v$ of $g$ which satisfies $g(\widehat v) - \min_z g(v) \le \eps$.
We denoted it by $\widehat v = \mathrm{GM}\left( (w_i)_{i=1}^m, (\alpha_i)_{i=1}^m, \eps \right)$.
Further, when $\alpha_i = 1/m$, we write
$\mathrm{GM}\left( (w_i)_{i=1}^m, \eps \right)$.

The GM has an optimal breakdown point of 1/2~\citep{lopuhaa1991breakdown}. 
That is, to get the geometric median to equal an arbitrary point, at least half the points (in total weight) must be modified.
We assume that $w_1, \cdots w_m$ are non-collinear, which is reasonable in the federated setting. 
Then, $g$ admits a unique minimizer $v^\star$.
Further, we assume $\sum_i \alpha_i = 1$ w.l.o.g.\footnote{
One could apply the results to $\widetilde g(v) := g(v) / \sum_{i=1}^m \alpha_i$.
} %

\myparagraph{Robust Federated Aggregation: The \rbfedavg Algorithm}
The \rbfedavg algorithm is obtained by replacing the mean aggregation of \fedavg with this GM-based robust aggregation oracle -- the full algorithm is given in~\Cref{algo:rfa:main}.
Similar to \fedavg, \rbfedavg also trades-off some communication for local computation by running multiple local steps in line~\ref{line:rfa:local}. The communication efficiency and privacy preservation of \rbfedavg follow from computing the GM as an iterative secure aggregate, which we turn to next. 
Note that \rbfedavg is agnostic to the \emph{actual} level of corruption in the problem and the aggregation is robust regardless of the convexity of the local objectives $F_i$.

\myparagraph{Geometric Median as an Iterative Secure Aggregate}
While the GM is a natural robust aggregation oracle, the key challenge in the federated setting is to implement it as an iterative secure aggregate. Our approach, given in \Cref{algo:rfa:weiszfeld}, 
iteratively computes a new weight $\beta_i\pow{r} \propto 1/\norm{v\pow{r} - w_i}$, up to a tolerance $\nu > 0$, whose role is to prevent division by zero. This endows the algorithm with greater stability. We call it the smoothed Weiszfeld algorithm as it is a variation of Weiszfeld's classical algorithm~\cite{weiszfeld1937point}. 
The smoothed Weiszfeld algorithm satisfies the following convergence guarantee, proved in Appendix~\ref{sec:a:weiszfeld}.

\begin{proposition} \label{prop:weiszfeld:main}
	The iterate $v\pow{R}$ of \Cref{algo:rfa:weiszfeld}
	with input $v\pow{0} \in \conv\{w_1, \cdots, w_m\}$ and $\nu > 0$ satisfies
	\[
		g(v\pow{R}) - g(v^\star) 
			\le \frac{2 \normsq{v\pow{0} - v^\star}}{\overline\nu R} + \frac{\nu}{2} \,,
	\]
	where $v^\star = \argmin g$ and 
	$\overline \nu =  \min_{r \in [R], i\in[m]} \nu \lor \norm{v\pow{r-1} - w_i} \ge \nu$.
	Furthermore, if 
	$0 < \nu \le \min_{i=1,\cdots, m} \norm{v^\star - w_i}$, then it holds that 
	$
		g(v\pow{R}) - g(v^\star) \le {2 \normsq{v\pow{0} - v^\star}}/{\overline \nu R} \,.
	$
\end{proposition}
For a $\eps$-approximate GM, we set $\nu = O(\eps)$
to get a $O(1/\eps^2)$ rate. However, if 
the GM $v^\star$ is not too close to any $w_i$, 
then the same algorithm automatically enjoys a faster $O(1/\eps)$ rate. The algorithm enjoys plausibly an even faster convergence rate \textit{locally}, and we leave this for future work.  

The proof relies on constructing a jointly convex surrogate $G: \reals^d \times \reals_{++}^m \to \reals$ 
defined using $\eta = (\eta_1, \cdots, \eta_m) \in \reals^m$ as 
\begin{align*} %
	G(v, \eta) := \frac{1}{2} \sum_{k=1}^m \alpha_k \left( \frac{\normsq{v - w_k}}{\eta_k} + \eta_k \right) \,.
\end{align*}
Instead of minimizing $g(v)$ directly using the equality $g(v) = \inf_{\eta > 0} G(v, \eta)$, we impose the constraint $\eta_i \ge \nu$ instead to avoid division by small numbers. 
The following alternating minimization leads to \Cref{algo:rfa:weiszfeld}:
\begin{align*}
	\eta\pow{r} = \argmin_{\eta \ge \nu} G(v\pow{r}, \eta) \,,
	\,  \text{and}, \,
	v\pow{r+1} = \argmin_{v \in \reals^d} G(v, \eta\pow{r}) \,. 
\end{align*}

Numerically, we find in \Cref{fig:expt:gm_algos_main}
that~\Cref{algo:rfa:weiszfeld} is rapidly convergent, 
giving a \textbf{high quality solution in 3 iterations}.
This ensures that the approximate GM as an iterative secure aggregate provides robustness at a modest 3$\times$ increase in communication cost over regular mean aggregation in \fedavg.

\begin{algorithm}[t]
	\caption{The Smoothed Weiszfeld Algorithm}
	\label{algo:rfa:weiszfeld}
\begin{algorithmic}[1]
		\Require $w_1, \cdots, w_m \in \reals^d$ with $w_i$ on device $i$, $\alpha_1, \cdots, \alpha_m > 0$, $\nu > 0$, budget $R$,
		$v\pow{0} \in \reals^d$,
	   		secure average oracle $\mathcal{A}$
	    \For{$r=0,1, \cdots, R-1$}
	    	\State Server broadcasts $v\pow{r}$ to devices $1, \cdots, m$
	    	\State Device $i$ computes $\beta\pow{r}_i = \alpha_i / (\nu \lor \norm{v\pow{r} - w_i})$ 
	   		\State $v\pow{r+1} \gets \left( {\sum_{i=1}^m \beta_i\pow{r} w_i} \right) / {\sum_{i=1}^m \beta_i\pow{r}}$
	   			using $\mathcal{A}$
	   			\label{line:weiszfeld:secure-average-oracle}
	    \EndFor
	    \Return $v\pow{R}$
\end{algorithmic}
\end{algorithm}

\myparagraph{Privacy Preservation}
While we can compute the geometric median as an iterate secure aggregate, privacy preservation also requires that the effective weights $\beta_i\pow{r} / \sum_{j} \beta_{j}\pow{r}$ are bounded away from 1 for each $i$. We show this holds for $m$ large.

\begin{proposition} \label{prop:rfa:privacy-preservation}
    Consider $\beta\pow{r}, v\pow{r}$ produced by \Cref{algo:rfa:weiszfeld} when
    given $w_1, \cdots w_m \in \reals^d$ with weights $\alpha_i = 1/m$ for each $i$ as inputs. Denote $B = \max_{i, j} \norm{w_i - w_{j}}$ and $\bar \nu$ as in \Cref{prop:weiszfeld:main}.
    Then, we have for all $i \in [m]$ and $r \in [R]$ that
    \[
        \frac{\beta_i\pow{r}}{\sum_{j=1}^m \beta_{j}\pow{r}} \le \frac{B}{B + (m-1)\bar\nu} \,.
    \]
\end{proposition}
\begin{proof}
    Since $v\pow{r} \in \conv\{w_1, \cdots, w_m\}$, 
    we have
    $\bar \nu \le \norm{v\pow{r} - w_i} \le B$. Hence, 
    $\alpha_i / B \le \beta_i\pow{r} \le \alpha_i / \bar \nu$ for each $i$ and $r$ and the proof follows.
\end{proof}

\begin{figure*}[tb!]
    \centering
    \includegraphics[width=\textwidth]{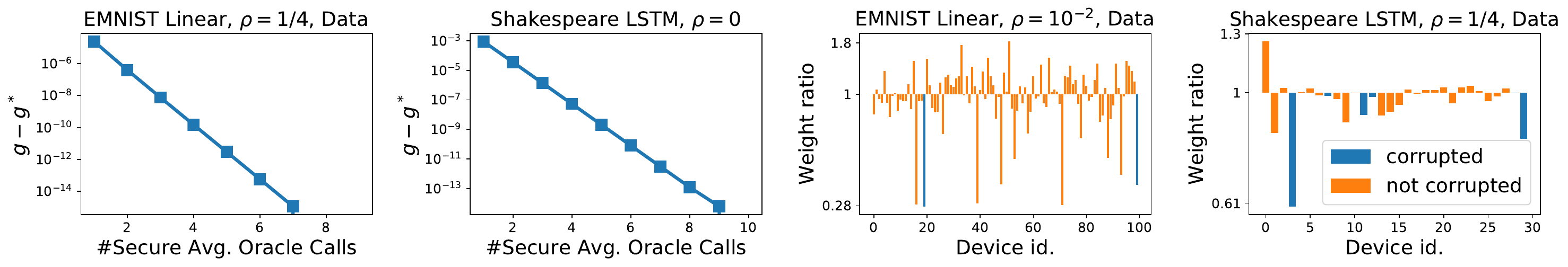}

     \caption{\small{\textbf{Left two}: Convergence of the smoothed Weiszfeld algorithm. \textbf{Right two}: Visualization of the re-weighting $\beta_i / \alpha_i$, 
     		where $\beta_i$ is the weight of $w_i$ in
			$\mathrm{GM}((w_i), (\alpha_i)) = \sum_i \beta_i w_i$.
	        See Appendix~\ref{sec:a:expt:gm_algos} for details.
     	}}
     	\label{fig:expt:gm_algos_main}
\end{figure*}

\subsection{Convergence Analysis of \rbfedavg}
We now present a convergence analysis of \rbfedavg
under two simplifying assumptions. 
First, we focus on least-squares fitting of additive models, as it allows us to leverage sharp analyses of SGD~\citep{bach2013non,jain2017markov,jain2017parallelizing}
and focus on the effect of the aggregation.
Second, we assume w.l.o.g. that each device is weighted by $\alpha_i = 1/n$ to avoid technicalities of random sums $\sum_{i \in S_t} \alpha_i$.
This assumption can be lifted with standard reductions; see~\Cref{remark:rfa:convergence}. 

\myparagraph{Setup}
We are interested in the supervised learning setting where $z_i \equiv (x_i, y_i) \sim D_i$ is an input-output pair. 
We assume that the output $y_i$ satisfies $\expect[y_i] = 0$ 
and $\expect[y_i^2] < \infty$.
Denote the marginal distribution of input $x_i$ as $D_{X, i}$.
The goal is to estimate the regression function $\overline{x} \mapsto \expect[y_i | x_i=\overline{x}]$
from a training sequence of independent copies of $(x_i, y_i) \sim D_i$ in each device. 
The corresponding objective is the square loss minimization
\begin{align} \label{eq:sgd:sq-loss}
F(w) &= \frac{1}{n} \sum_{i=1}^n F_i(w) \;,
	\quad \text{where} \quad
	F_i(w) = \frac{1}{2}\,\expect_{(x,y) \sim \new{D_i}} \left(y - w\T\phi(x)\right)^2\; \text{for all } i\in[n]\,.
\end{align}
Here,
$\phi(x) = (\phi_1(x), \dots, \phi_d(x)) \in \reals^d$
where $\phi_1, \dots, \phi_d$ 
are a fixed basis of measurable, centered functions. 
The basis functions may be nonlinear, thus encompassing
random feature approximations of kernel feature maps and 
pre-trained deep network feature representations.

We state our results under the following assumptions:
(a) the feature maps are bounded as $\norm{\phi(x)} \le R$ with probability one under $D_{X, i}$
for each device $i$; 
(b) each $F_i$ is $\mu$-strongly convex; 
(c) the additive model is well-specified on each device: for each device $i$,
there exists $w_i^\star \in \reals^d$ such that $y_i = \phi(x_i)\T w_i^\star + \zeta_i$  where $\zeta_i \sim \mathcal{N}(0, \sigma^2)$.
The second assumption is equivalent to 
requiring that $H_i = \grad^2 F_i(w) = \expect_{x \sim D_{X, i}}[\phi(x)\phi(x)\T]$, the covariance of $x$ on device $i$ has eigenvalues no smaller than $\mu$.

\myparagraph{Quantifying Heterogeneity}
We quantify the heterogeneity in the data distributions $D_i$ across devices in terms of the heterogeneity of marginals $D_{X, i}$ and of the conditional expectation $\expect[y_i | x_i = x] = \phi(x)\T w_i^\star$.
Let $H = \grad^2 F(w) = (1/n) \sum_{i=1}^n H_i$ be the covariance of $x$ under the mixture distribution across devices, where $H_i$ is the covariance of $x_i$ in device $i$. We measure the dissimilarities $\niidterm_X, \niidterm_{Y|X}$ of the marginal and the conditionals respectively as
\begin{align}
	\label{eq:rfa:niid_xy}
	\niidterm_X &= \max_{i \in [n]} \lambda_{\max}(H^{-1/2} H_i H^{-1/2})\,, \quad \text{and}, \quad
	\niidterm_{Y|X} = \max_{i, j \in [n]} \norm{w_i^\star - w_{j}^\star} \,,
\end{align}
where $\lambda_{\max}(\cdot)$ denotes the largest eigenvalue.
Note that $\niidterm_X \ge 1$ and it is equal to $1$ iff each $H_i = H$. It measures the 
spectral misalignment between each $H_i$ and $H$. 
The second condition is related to the Wasserstein-2 distance~\cite{panaretos2020invitation} between 
the conditionals $D_{Y|X, i}$ as
$W_2(D_{Y|X, i}, D_{Y|X, j}) \le R \niidterm_{Y|X}$.
We define the degree of heterogeneity 
between the various $D_i = D_{X, i} \otimes D_{Y|X, i}$ as 
$\mathrm{width}(\mathcal{D}) = \niidterm_{X} \niidterm_{Y|X} =: \niidterm$. 
That is, if the conditionals are the same ($\niidterm_{Y|X} = 0$), 
we can tolerate arbitrary heterogeneity in the marginals $D_{X, i}$.

\myparagraph{Convergence}
We now analyze \rbfedavg where the local SGD updates are equipped with 
``tail-averaging''~\citep{jain2017parallelizing}
so that $w_i\pow{t+1} = (2/\tau)\sum_{k=\tau/2}^\tau w_{i, k}\pow{t}$
is averaged over the latter half of the trajectory of iterates instead of line~\ref{line:algo:rfa:local-update-final} of \Cref{algo:rfa:main}.
We show that this variant of \rbfedavg converges up to the dissimilarity level $\niidterm = \niidterm_X \niidterm_{Y|X}$ when the corruption level $\rho < 1/2$.

\begin{theorem}\label{thm:rfa:convergence}
	Consider $F$ defined in \eqref{eq:sgd:sq-loss} and suppose the corruption level satisfies $\rho < 1/2$.
	Consider~\Cref{algo:rfa:main} run for $T$ outer iterations 
	with a learning rate $\gamma = 1/(2R^2)$, and the local updates are
	run for $\tau_t$ steps in outer iteration $t$ with tail averaging. 
	Fix $\delta > 0$ and $\theta \in (\rho, 1/2)$, and set the number of devices per iteration, $m$ as
	\begin{align} \label{eq:rfa:mainthm:ndev-per-round}
		m \ge \frac{\log(T/\delta)}{2(\theta - \rho)^2} \,.
	\end{align}
	Define $C_\theta := (1-2\theta)^{-2}$, $w^\star = \argmin F$,
	$F^\star = F(w^*)$,
	$\kappa := R^2/\mu$ and
	$\Delta_0 := \normsq{w\pow{0} - w^\star}$.
	Let $\tau \ge 4\kappa \log\left(128 C_\theta \kappa \right)$.
	We have that the event $\mathcal{E} = \bigcap_{t=0}^{T-1}\{ |S_t \cap \mathcal{C}| \le \theta m\}$
	holds with probability at least $1-\delta$.
	Further, if %
	$\tau_t = 2^t \tau$ for each iteration $t$, 
	then the output $w\pow{T}$ of 
	\Cref{algo:rfa:main} satisfies, 
	\[
		\expect\left[\normsq{w\pow{T}) - w^\star} \,\middle\vert\, \mathcal{E}\right] \le 
		\frac{\Delta_0}{2^{T}} 
			+ C C_\theta \left(\frac{d \sigma^2 T}{ \mu \tau 2^{T}}
			+ \frac{\eps^2}{m^2} + \niidterm^2 \right)
	\]
	where $C$ is a universal constant.
	If $\tau_t = \tau$ instead, then, 
	the noise term above reads $d\sigma^2 / {\mu \tau}$.
\end{theorem}
\Cref{thm:rfa:convergence} shows near-linear convergence $O(T / 2^{T})$ up to two error terms in the case that $\rho$ is bounded away from 
$1/2$ (so that $\theta$ and $C_\theta$ can be taken to be constants).
The increasing local computation $\tau_t = 2^t \tau$ required by this rate is feasible since local computation is assumed to be cheaper than communication.

The first error term is $\eps^2/m^2$ due to approximation $\eps$ in the GM, which can be made arbitrarily small by increasing the number $m$ of devices sampled per round.
The second error term $\niidterm^2$ is due to heterogeneity. 
Indeed, exact convergence as $T\to \infty$ is not possible in the presence of corruption: lower bounds for robust mean estimation~\cite[e.g.][Theorem 2.2]{chen2018robust} imply that
$\normsq{w\pow{T} - w^\star} \ge C \rho^2 \niidterm_{Y|X}^2$
w.p. at least $1/2$.
Consistent with our theory, we find in real heterogeneous datasets in \Cref{sec:expt} that \rbfedavg can lead to \emph{marginally} worse performance than \fedavg in the corruption-free regime ($\rho = 0$).
Finally, while we focus on the setting of least squares, our results can be extended to the general convex case. 

\begin{remark} \label{remark:rfa:convergence}
		For unequal weights, we can perform the reduction 
			$\tilde F_i(w) = n\alpha_i F_i(w)$, 
			so the theory applies with the substitution
			$(R^2, \sigma^2, \mu, \niidterm_X) \mapsto (c_1 \, R^2$, 
			$c_1 \, \sigma^2$, $c_2 \, \mu, (c_1/c_2) \niidterm_X)$,
			where $c_1 = n \max_i \alpha_i$ and $c_2 = n \min_i \alpha_i$.
\end{remark}
\noindent We use the following convergence result of SGD~\citep[Theorem~1]{jain2017markov},~\citep[Corollary~2]{jain2017parallelizing}.
\begin{theorem}[\citep{jain2017parallelizing,jain2017markov}]
\label{thm:sgd:central}
	Consider a $F_k$ from~\eqref{eq:sgd:sq-loss}. 
	Then, defining $\kappa := R^2/\mu$, the output $\overline v_{\tau}$ of $\tau$ steps of tail-averaged SGD starting from $v_0 \in \reals^d$ using learning rate $(2R^2)^{-1}$ satisfies
	\[
		\expect\normsq{\overline v_{\tau} - w^\star} \le 
			2\kappa \exp\left( -\frac{\tau}{4\kappa}  \right) \normsq{ v_0 - w^\star} 
			+ \frac{8d\sigma^2}{\mu \tau} \,.
	\]
\end{theorem}
\begin{proof}[Proof of \Cref{thm:rfa:convergence}]
    Define the event $\mathcal{E}_t = \{ |S_t \cap \mathcal{C}| \le \theta m\}$
    so that $\mathcal{E} = \bigcap_{t=0}^{T-1} \mathcal{E}_t$.
    Hoeffding's inequality gives $\prob(\overline{\mathcal{E}_t}) \le \delta/T$ for each $t$ so that 
    $\prob(\overline {\mathcal{E}}) \le \delta$ using the union bound.
    Below, let $\mathcal{F}_t$ denote the sigma algebra generated by $w\pow{t}$.
    
    Consider the local updates on an uncorrupted device $i \in S_t \setminus \mathcal{C}$, 
    starting from $w\pow{t}$. Theorem~\ref{thm:sgd:central} gives, upon
    using $\tau_t \ge \tau \ge 4\kappa \log(128 C_\theta \kappa)$,
    \[
    	\expect\left[ \normsq{w\pow{t+1}_i - w^\star_i} \, \middle| \, \mathcal{E}, \mathcal{F}_t \right]
    		\le \frac{1}{64 C_\theta} \normsq{w\pow{t} - w^\star_i} 
    		+ \frac{8 d \sigma^2}{\mu \tau_t} \,.
    \]
    Note that $w^\star = (1/n) \sum_{j=1}^n H\inv H_{j} w_{j}^\star$, so that 
    \[
    	\norm{w^\star - w_i^\star} \le \frac{1}{n}\sum_{j=1}^n \norm{H\inv H_{j}(w_{j}^\star - w_i^\star)} \le \niidterm \,.
    \] 
    Using $\normsq{a+b} \le 2\normsq{a} + 2\normsq{b}$, we get, 
    \begin{align*}
    	\expect\Big[ \normsq{w\pow{t+1}_i - w^\star} \, \big| & \, \mathcal{E}, \mathcal{F}_t \Big]
    	\le 2\expect\left[ \normsq{w\pow{t+1}_i - w^\star_i} \, \middle| \, \mathcal{E}, \mathcal{F}_t \right] + 2\niidterm^2 \\
    	& \le \frac{1}{32C_\theta} \normsq{w\pow{t} - w^\star_i} 
    		+ \frac{16 d \sigma^2}{\mu \tau_t}  + 2\niidterm^2 \\ 
    	&\le \frac{q}{16C_\theta} \normsq{w\pow{t} - w^\star} 
    		+ \frac{16 d \sigma^2}{\mu \tau_t}  + 4\niidterm^2 \,.
    \end{align*}
    We now apply 
    the robustness property of the GM~(\citep[Thm. 2.2]{lopuhaa1991breakdown} or ~\citep[Lem.~3]{wu2020federated})
    to get, 
    \[
    	\expect\left[ \normsq{w\pow{t+1} - w^\star} \, \middle| \, \mathcal{E}, \mathcal{F}_t \right]
    		\le \frac{1}{2} \normsq{w\pow{t} - w^\star} 
    			+ \frac{128 C_\theta d \sigma^2}{\mu \tau_t}  + \Gamma \,,
    \]
    where $\Gamma = 2C_\theta (\eps^2/m^2 + 16\niidterm^2)$.
    Taking an expectation conditioned on $\mathcal{E}$ and unrolling this inequality gives
    \[
    	\expect\left[ \normsq{w\pow{T} - w^\star} \, \middle| \, \mathcal{E} \right]
    	\le 
    	\frac{\Delta_0}{2^T} + \frac{128C_\theta d\sigma^2}{\mu} \sum_{t=1}^T \frac{1}{2^{T-t} \tau_t} + 2 \Gamma \,.
    \]
    When $\tau_t = 2^t \tau$, 
    the series sums to $2^{-(T-1)} T/\tau$, while for $\tau_t = \tau$, the series is upper bounded by $2/\tau$.
\end{proof}

We now consider \rbfedavg in connection with the three factors mentioned in \Cref{sec:setup:fl}.
\begin{enumerate}[label=(\roman*),nolistsep,leftmargin=\widthof{ (a) }]
    \item \textbf{Communication Efficiency}: Similar to \fedavg, \rbfedavg performs multiple local updates for each aggregation round, to save on the total communication. However, owing to the trade-off between communication, privacy and robustness, \rbfedavg requires a modest 3$\times$ more communication for robustness per aggregation. In the next section, we present a heuristic to reduce this communication cost to one secure average oracle call per aggregation.
    
    \item \textbf{Privacy Preservation}: \Cref{algo:rfa:weiszfeld} computes the aggregation as an iterative secure aggregate.
    This means that the server only learns the intermediate parameters after being averaged over all the devices, with effective weights bounded away from $1$ (\Cref{prop:rfa:privacy-preservation}). The noisy parameter vectors sent by individual devices are uniformly uninformative in information theoretic sense with the use of secure multi-party computation.
    
    \item \textbf{Robustness}: The geometric median has a breakdown point of 1/2~\cite[Theorem~2.2]{lopuhaa1991breakdown}, which is the highest possible~\cite[Theorem~2.1]{lopuhaa1991breakdown}. In the federated learning context, this means that convergence is still guaranteed by \Cref{thm:rfa:convergence} when up to half the points in terms of total weight are corrupted. \rbfedavg is resistant to both data or update poisoning, while being privacy preserving. On the other hand, \fedavg has a breakdown point of 0, where a single corruption in each round can cause the model to become arbitrarily bad.
\end{enumerate}

\begin{algorithm}[t!]
	\caption{One-step Smoothed Weiszfeld Algorithm}
	\label{algo:rfa:weiszfeld-one-shot}
\begin{algorithmic}[1]
        \Require Same as \Cref{algo:rfa:weiszfeld}
        \State Device $i$ sets $\beta_i = \alpha_i / (\nu \lor \norm{w_i})$
        \State \Return $\left( {\sum_{i=1}^m \beta_i w_i} \right) / {\sum_{i=1}^m \beta_i}$ using $\mathcal{A}$
\end{algorithmic}
\end{algorithm}

\begin{algorithm}[t!]
	\caption{\rbfedavg with Personalization}
	\label{algo:rfa:res}
\begin{algorithmic}[1]
        \Statex Replace lines~\ref{line:rfa:local-update-start} to~\ref{line:algo:rfa:local-update-final} of \Cref{algo:rfa:main} with the following:
        \State Set $u_{i,0}\pow{t}=u_i\pow{t}$ and $w_{i,0}\pow{t} = w\pow{t}$
        \For{$k=0, \cdots, \tau-1$}
            \State $u_{i, k+1}\pow{t} = u_{i, k}\pow{t} - \gamma \grad f(w\pow{t} + u_{i, k}\pow{t} ; z_{i, k}\pow{t})$ with 
            $z_{i, k}\pow{t} \sim D_i$
        \EndFor
        \For{$k=0, \cdots, \tau-1$}
            \State $w_{i, k+1}\pow{t} = w_{i, k}\pow{t} - \gamma \grad f(w_{i, k}\pow{t} + u_{i, \tau}\pow{t} ; \tilde z_{i, k}\pow{t})$ with 
            $\tilde z_{i, k}\pow{t} \sim D_i$
        \EndFor
        \State Set $w_i\pow{t+1} = w_{i, \tau}\pow{t}$ and $u_i\pow{t+1} = u_{i, \tau}\pow{t}$
\end{algorithmic}
\end{algorithm}

\subsection{Extensions to \rbfedavg}
We now discuss two extensions to \rbfedavg to reduce the communication cost (without sacrificing privacy) and better accommodate statistical heterogeneity in the data with model personalization. 

\myparagraph{One-step \rbfedavg: Reducing the Communication Cost}
Recall that \rbfedavg results in a 3-5$\times$ increase in the communication cost over \fedavg. Here, we give a heuristic variant of \rbfedavg in an extremely communication-constrained setting, where
it is infeasible to run multiple iterations of \Cref{algo:rfa:weiszfeld}. We simply run \Cref{algo:rfa:weiszfeld} with $v\pow{0} = 0$ and a communication budget of $R=1$; see \Cref{algo:rfa:weiszfeld-one-shot} for details.
We find in \Cref{sec:expt:extensions} that one-step \rbfedavg retains most of the robustness of \rbfedavg. 

\myparagraph{Personalized \rbfedavg: Offsetting Heterogeneity}
We now show \rbfedavg can be extended to better handle heterogeneity 
in the devices with the use of personalization.
The key idea is that predictions are made on device $i$ by summing 
the shared parameters $w$ maintained by the server
with personalized parameters $U = \{u_1, \cdots, u_n\}$ 
maintained individually on-device.
In particular, the optimization problem we are interested in solving is 
\begin{align*}
	 \min_{w, U} \left[ F(w, U) := 
	 \sum_{i=1}^n \alpha_k \, \expect_{z \sim D_i} 
	 \left[ f(w + u_i ; z) \right]
	 \right] \,.
\end{align*}

We outline the algorithm in \Cref{algo:rfa:res}.
We train the shared and personalized parameters on each other's residuals,
following the residual learning scheme of~\cite{agarwal2020federated}.
Each selected device first updates its personalized parameters $u_i$ 
while keeping the shared parameters $w$ fixed. 
Next, the updates to the shared parameter are computed on the residual of the personalized parameters.
The updates to the shared parameter are aggregated with the geometric median, identical to \rbfedavg. Experiments in \Cref{sec:expt:extensions} show that personalization is effective in combating heterogeneity. 

\begin{table*}[b!]
\caption{Dataset description and statistics.}
\label{table:expt:dataset:descr}
\begin{center}
\begin{adjustbox}{max width=0.9\linewidth}
\begin{tabular}{lcccccccc}
\toprule
Dataset & Task & \#Classes & \#Train & \#Test & \#Devices & \multicolumn{3}{c}{\#Train per Device} \\
 & & & & & & Median & Max & Min  \\
\midrule

EMNIST & \begin{tabular}{c} Image Classification \end{tabular} & 
    62 & $204K$ & $23K$ & 1000 & 160 & 418 & 92 \\

Shakespeare & \begin{tabular}{c} Character-level Language Modeling \end{tabular} & 
    53 & $2.2 M$ & $0.25 M$ &  628 & 1170 & 70600 & 90 \\

\new{Sent140} & \begin{tabular}{c} \new{Sentiment Analysis} \end{tabular} & 
    \new{2} & \new{$57 K$} & \new{$15 K$} &  \new{877} & \new{55} & \new{479} & \new{40} \\

\bottomrule
\end{tabular}
\end{adjustbox}
\end{center}
\vskip -0.1in
\end{table*}

\begin{figure*}[t!]
    \centering
        \includegraphics[width=0.99\textwidth,trim={0 34pt 0 0},clip=true]{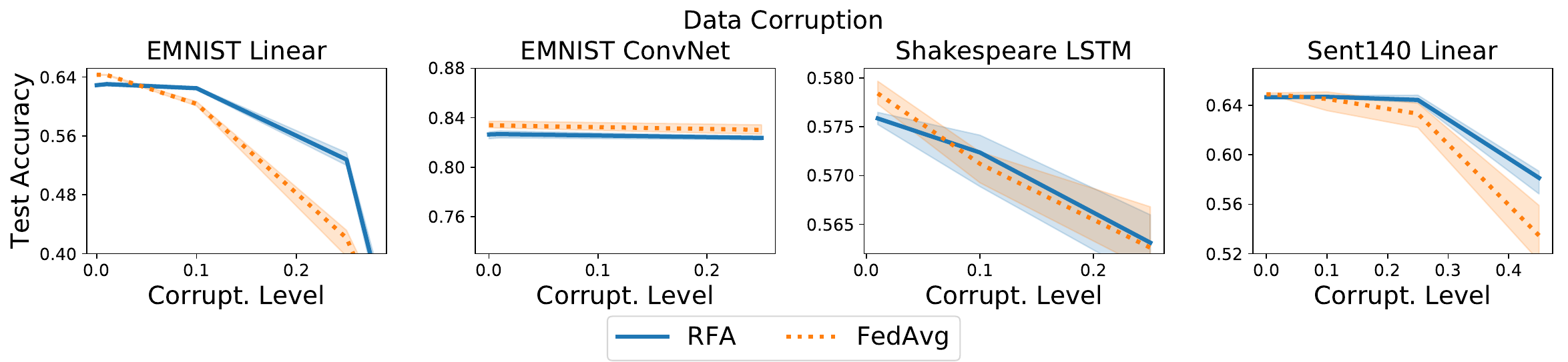}

        \includegraphics[width=0.99\textwidth]{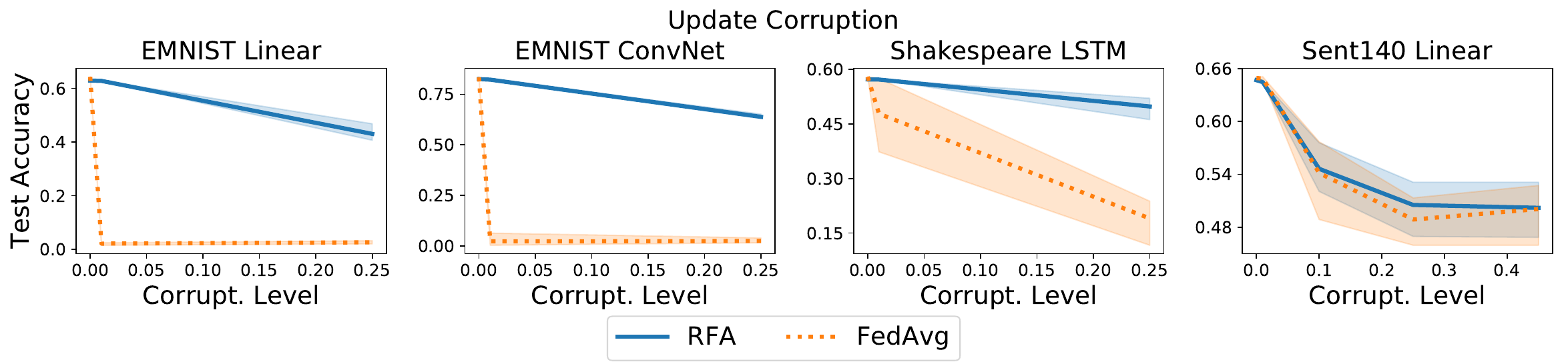}
     \caption{\new{Comparison of robustness of \rbfedavg{} and \fedavg{} under data corruption (\textbf{top}) and update corruption (\textbf{bottom}).
     The left three plots for update corruption show omniscient corruption while the rightmost one shows Gaussian corruption. The shaded area denotes minimum and maximum over 5 random seeds.}} 
     \label{fig:expt:robustness_main}
\end{figure*}

\section{Numerical Simulations} \label{sec:expt}
We now conduct simulations to compare \rbfedavg with other federated learning algorithms. The simulations were run using TensorFlow and the data was preprocessed using LEAF~\citep{caldas2018leaf}. 
We first describe the experimental setup in \Cref{sec:expt:setup}, 
then study the robustness and convergence of \rbfedavg in \Cref{sec:expt:robustnesss}.
We study the effect of the extensions of \rbfedavg in \Cref{sec:expt:extensions}.
The full details from this section and more simulation results are given in Appendix~\ref{sec:a:expt}. 
The code and scripts to reproduce these experiments can be found online~\cite{rfa_repo}.

\subsection{Setup} \label{sec:expt:setup}

We consider three machine learning tasks. 
The datasets are described in Table~\ref{table:expt:dataset:descr}.
As described in \Cref{sec:setup:fl}, we take the weight $\alpha_i$ of device $i$ to be proportional to the number of datapoints $N_i$ on the device. 

\begin{enumerate}[label=(\alph*),nolistsep,leftmargin=\widthof{ (a) }]
\item{\textit{Character Recognition}:}
We use the EMNIST dataset~\citep{cohen2017emnist}, 
where the input $x$ is a $28\times 28$ grayscale image of a handwritten character
and the output $y$ is its identification (0-9, a-z, A-Z). 
Each device is a writer of the handwritten character $x$.
We use two models --- a linear model $\varphi(x;w) = w\T x$ and 
a convolutional neural network (ConvNet).
We use as objective $f(w; (x, y)) = \ell(y, \varphi(x;w))$, where 
$\ell$ is the multinomial logistic loss $\ell$. 
We evaluate performance using the classification accuracy. 

\item{\textit{Character-Level Language Modeling}:}
We learn a character-level language model over the 
Complete Works of Shakespeare~\cite{shakespeare}. 
We formulate it as a multiclass classification problem, where 
the input $x$ is a window of 20 characters, the output $y$ is the next (i.e., 21st) character.
Each device is a role from a play (e.g., Brutus from The Tragedy of Julius Caesar).
We use a long-short term memory model (LSTM)~\citep{hochreiter1997long} 
together with the multinomial logistic loss. The performance is evaluated with the classification accuracy of next-character prediction.

\item{\textit{Sentiment Analysis}:}
We use the Sent140
dataset~\citep{go2009twitter}
where the input $x$ is a tweet
and the output $y=\pm1$ is its sentiment.
Each device is a distinct Twitter user.
We use a linear model using average of the GloVe embeddings~\cite{pennington2014glove} of the words of the tweet. It is trained with the binary logistic loss and evaluated with the classification accuracy.
\end{enumerate}

\myparagraph{Corruption Models}
We consider the following corruption models for
corrupted devices $\mathcal{C}$, cf. \Cref{sec:setup:tradeoffs}:
\begin{enumerate}[label=(\alph*),nolistsep,leftmargin=\widthof{ (a) }]
\item{\textit{Data Poisoning}}: The distribution $D_i$
    on a device $k \in \mathcal{C}$ is replaced by some fixed $\tilde D_i$.
	For EMNIST, we take the negative of an image so that
    $\tilde D_i(x, y) = D_i(1-x, y)$.
	For the Shakespeare dataset, we reverse the text so that 
    $\tilde D_i(c_1,\cdots c_{20}, c_{21}) = D_i(c_{21},\cdots c_2, c_1)$.
    In both these cases, the labels are unchanged.
    For the Sent140 dataset, we flip the label while keeping $x$ unchanged.

\item \textit{Update poisoning with Gaussian corruption}:
    Each corrupted device $i \in \mathcal{C}$ 
    returns $w_i\pow{t+1} = w_{i,\tau}\pow{t} + \zeta_i\pow{t}$, where $\zeta_{i}\pow{t} \sim \mathcal{N}(0, \sigma^2 I)$, where $\sigma^2$ is the variance across the components of $w_{i,\tau}\pow{t}-w\pow{t}$.\footnote{
    Model updates $w_i\pow{t}-w\pow{t}$ are aggregated, not the models $w_i\pow{t}$ directly~\cite{kairouz2019flsurvey}.}

\item \textit{Update poisoning with omniscient corruption}:
    The parameters $w_i\pow{t+1}$ returned by devices $i \in \mathcal{C}$ 
	are modified so that the weighted arithmetic mean $\sum_{i \in S_t} \alpha_i w_i\pow{t+1}$ 
    over the selected devices $S_t$ is set to $-\sum_{i \in S_t} \alpha_i w_{i, \tau}\pow{t}$, the negative of 
    what it would to have been without the corruption.
    This is designed to hurt the weighted arithmetic mean aggregation.
\end{enumerate}

\myparagraph{Hyperparameters}
The hyperparameters %
are chosen similar to the defaults of~\citep{mcmahan2017communication}. 
A learning rate schedule was tuned on a validation set for \fedavg with no corruption.
The same schedule was used for \rbfedavg.
The aggregation in \rbfedavg is implemented using the smoothed Weiszfeld algorithm 
with a budget of $R=3$ calls to the secure average oracle, 
thanks to its rapid empirical convergence (cf. \Cref{fig:expt:gm_algos_main}),
and $\nu = 10^{-6}$ for numerical stability.
Each simulation was repeated 5 times and the shaded area denotes the minimum and maximum over these runs.
Appendix~\ref{sec:a:expt} gives details on hyperparameter, and
a sensitivity analysis of the Weiszfeld communication budget.

\subsection{Robustness and Convergence of \rbfedavg} \label{sec:expt:robustnesss}

First, we compare the robustness of \rbfedavg as opposed to vanilla \fedavg to different types of corruption across different datasets in
\Cref{fig:expt:robustness_main}. We make the following observations.

\myparagraphsmall{\rbfedavg gives improved robustness to linear models with data corruption}
For instance, consider the EMNIST linear model at $\rho=1/4$. \rbfedavg achieves 52.8\% accuracy, over 10\% better than \fedavg at 41.2\%.

\myparagraphsmall{\rbfedavg performs similarly to \fedavg in deep nets with data corruption}
\rbfedavg and \fedavg are within one standard deviations of each other for the Shakespeare LSTM model, and nearly equal for the EMNIST ConvNet model. 
We note that the behavior of the training of a neural network when the data is corrupted is not well-understood in general~\citep[e.g.,][]{zhang2017understanding}.

\myparagraphsmall{\rbfedavg gives improved robustness to omniscient corruptions for all models}
For the omniscient corruption, the test accuracy of the \fedavg is close to 0\% for the EMNIST linear model and ConvNet, while \rbfedavg still achieves over 40\% at $\rho=1/4$ for the former and well over 60\% for the latter. A similar trend holds for the Shakespeare LSTM model.

\begin{figure}[t]
\centering
    \includegraphics[width=0.5\linewidth]{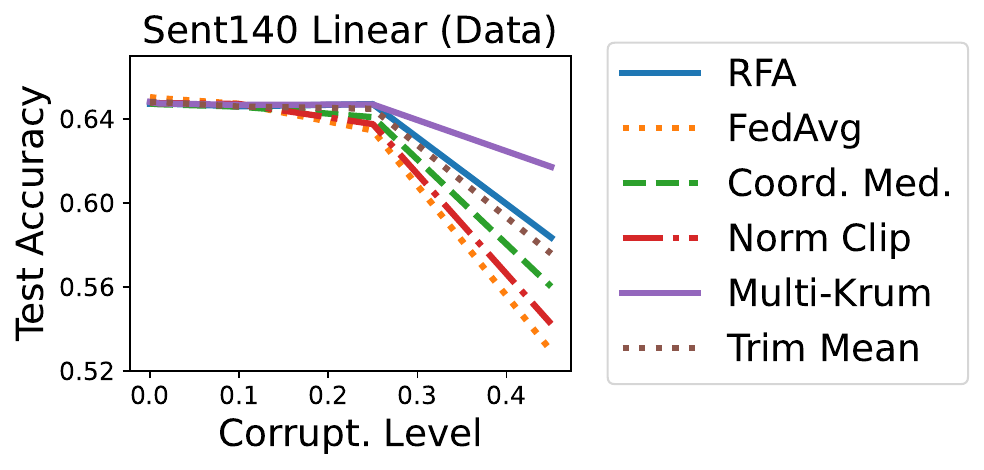}
    \caption{Comparison of \rbfedavg with other robust aggregation algorithms on Sent140 with data corruption.}
    \label{fig:rfa:robustness-other-agg}
\end{figure}

\begin{figure*}[t!]
    \centering %
        \adjincludegraphics[width=0.91\linewidth, trim={0 34pt 0 0},clip=true]{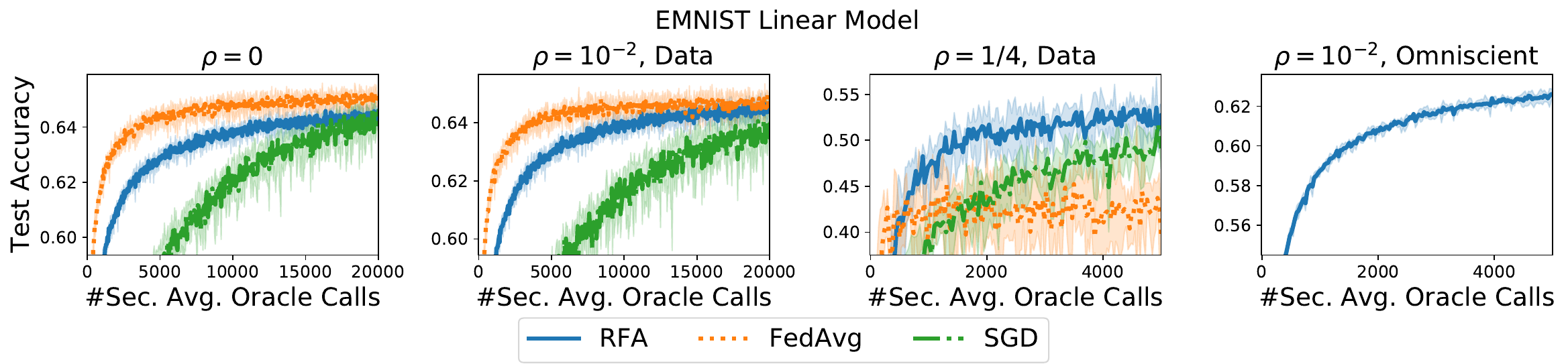}

        \adjincludegraphics[width=0.91\linewidth, trim={0 0 0 0},clip=true]{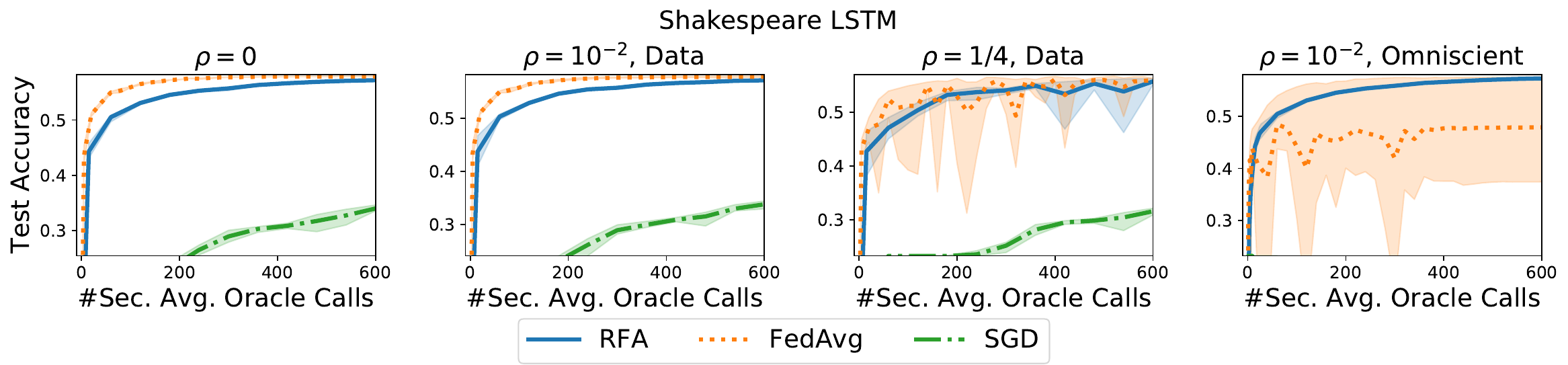}

     \caption{Comparison of methods plotted against number of calls to the secure average oracle for different corruption settings. For the case of omniscient corruption, \fedavg and SGD are not shown in the plot if they diverge. The shaded area denotes the maximum and minimum over 5 random seeds.}
     \label{fig:expt:main-plot:main}
\end{figure*}

\myparagraphsmall{\rbfedavg almost matches \fedavg in the absence of corruption}
Recall from \Cref{sec:setup:tradeoffs} that robustness comes at the cost of heterogeneity; 
this is also reflected in the theory of  \Cref{sec:aggr}. 
Empirically, we find that the performance hit of \rbfedavg due to heterogeneity is quite small: 1.4\% for the EMNIST linear model (64.3\% vs. 62.9\%), under 0.4\% for the Shakespeare LSTM, and 
0.3\% for Sent140 (65.0\% vs. 64.7\%). Further, we demonstrate in Appendix~\ref{sec:a:additional-results} that, consistent with the theory, this gap completely vanishes in the i.i.d. case. 

\begin{figure}[t!]
    \centering %
        \adjincludegraphics[width=0.6\linewidth, trim={0 0 0 0},clip=true]{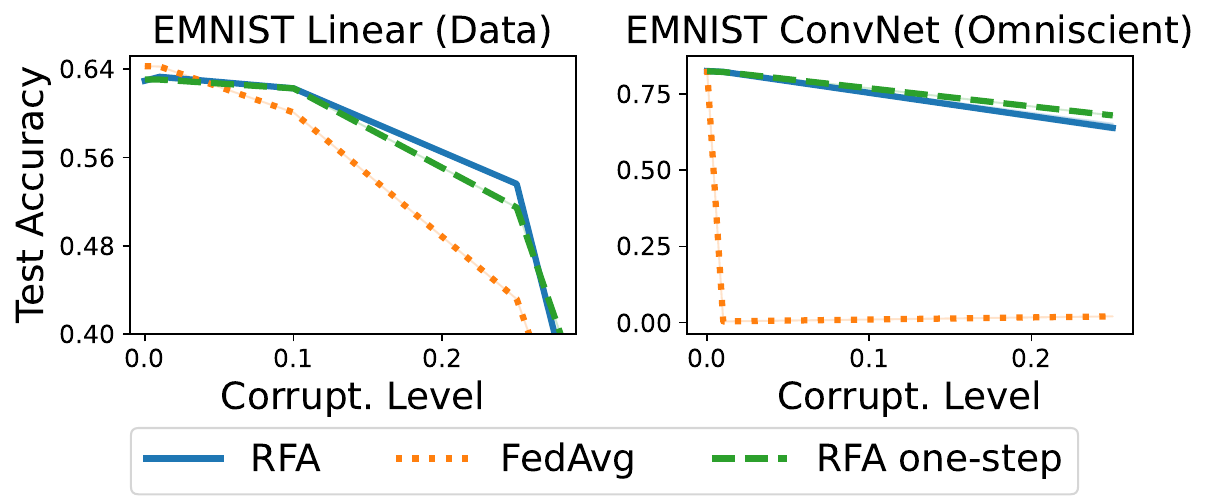}

     \caption{Robustness of one-step \rbfedavg.
     }
     \label{fig:expt:main-plot:one-shot}
\end{figure}

\myparagraphsmall{\rbfedavg is competitive with other robust aggregation schemes while being privacy-preserving}
We now compare \rbfedavg with: 
(a) coordinate-wise median~\cite{yin2018byzantine} and $\ell_2$ norm clipping~\cite{sun2019backdoor} which are agnostic to the actual corruption level $\rho$ like \rbfedavg, and, 
(b) trimmed mean~\cite{yin2018byzantine} and multi-Krum~\cite{blanchard2017machine}, that require exact knowledge of the level of corruption $\rho$ in the problem. We find that \rbfedavg is more robust than the two agnostic algorithms coordinate-wise median and norm clipping. Perhaps surprisingly, \rbfedavg is also more robust than the trimmed mean which uses perfect knowledge of the corruption level $\rho$. 
We note that multi-Krum is more robust than \rbfedavg. 
That being said, \rbfedavg has the advantage that it is fully agnostic to the actual corruption level $\rho$ and is privacy-preserving, while the other robust approaches are not. 

\myparagraphsmall{Summary: robustness of \rbfedavg}
Overall, we find that \rbfedavg is no worse than \fedavg in the presence of corruption and is often better, while being almost as good in the absence of corruption. Furthermore, 
\rbfedavg degrades more gracefully as the corruption level increases.

\myparagraphsmall{\rbfedavg requires only $3\times$ the communication of \fedavg}
Next, we plot in~\Cref{fig:expt:main-plot:main} 
the performance versus the number of rounds of
communication as measured by the number of calls to the secure average oracle.
We note that in the low corruption regime of 
 $\rho = 0$ or $\rho = 10^{-2}$ under data corruption, 
 \rbfedavg requires $3\times$ the number of calls to the secure average oracle to reach the same performance. 
 However, it matches the performance of \fedavg when measured in terms of the number of outer iterations, with the additional communication cost coming from multiple Weiszfeld iterations for computation of the average.

\myparagraphsmall{\rbfedavg exhibits more stable convergence under corruption}
We also see from~\Cref{fig:expt:main-plot:main} ($\rho=1/4$, Data) that the variability of accuracy across random runs, denoted here by the shaded region, is much smaller for \rbfedavg. Indeed, by being robust to the corrupted updates sent by random sampling of corrupted clients, \rbfedavg exhibits a more stable convergence across iterations.

\begin{figure}[t!]
    \centering %
        \adjincludegraphics[width=0.6\linewidth, trim={0 0 0 25pt},clip=true]{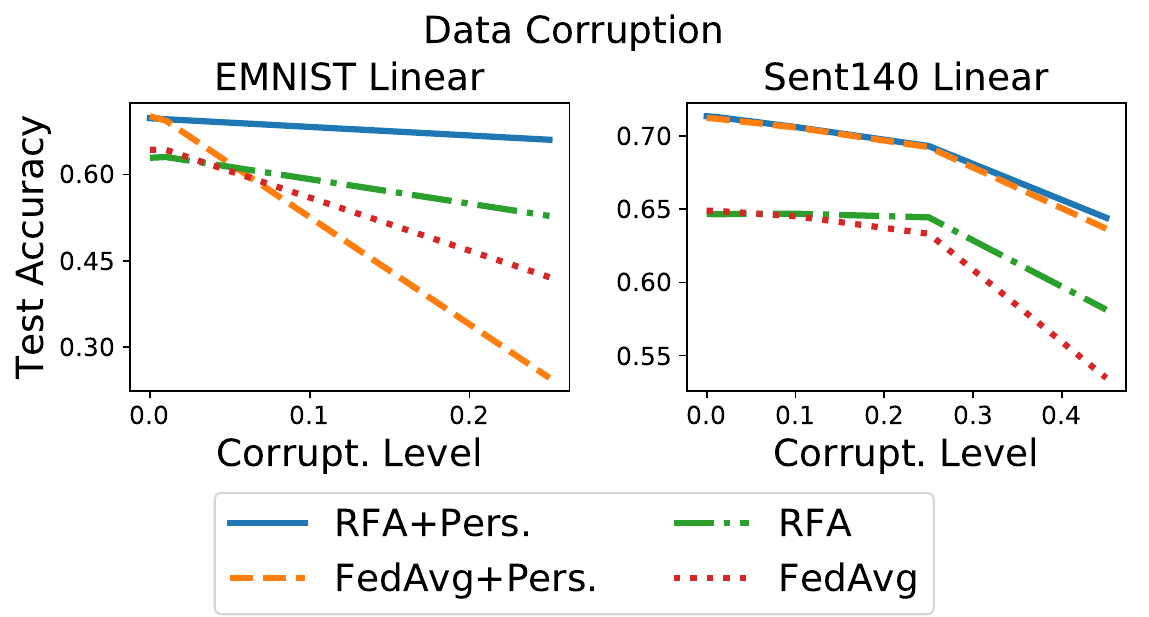}

     \caption{Effect of personalization on the robustness of \rbfedavg and \fedavg under data corruption.
     }
     \label{fig:expt:main-plot:personalization}
\end{figure}

\subsection{Extensions of \rbfedavg} \label{sec:expt:extensions}
We now study the proposed extensions: one-step \rbfedavg and personalization. 

\myparagraphsmall{One-step \rbfedavg gives most of the robustness with no extra communication} 
From~\Cref{fig:expt:main-plot:one-shot},
we observe that for one-step \rbfedavg
is quite close in performance to \rbfedavg 
across different levels of corruption for both
data corruption on an EMNIST linear model and omniscient corruption on an EMNIST ConvNet. 
For instance, in the former, one-step \rbfedavg gets 51.4\% in accuracy, which is 10\% better than \fedavg while being almost as good as full \rbfedavg (52.8\%) at $\rho=0.25$.
Moreover, for the latter, we find that one-step \rbfedavg (67.9\%) actually achieves higher test accuracy than full \rbfedavg (63.0\%) at $\rho=0.25$.

\myparagraphsmall{Personalization helps \rbfedavg offset effects of heterogeneity}
\Cref{fig:expt:main-plot:personalization} plots the effect of \rbfedavg with personalization. First, we observe that personalization leads to an improvement with no corruption for both \fedavg and \rbfedavg. For the EMNIST linear model, we get 70.1\% and 69.9\% respectively from 64.3\% and 62.9\%. 
Second, we observe that \rbfedavg exhibits greater robustness to corruption with personalization. At $\rho=1/4$ with the EMNIST linear model, \rbfedavg with personalization gives 66.4\% (a reduction of 3.4\%) while no personalization gives 52.8\% (a reduction of 10.1\%).
The results for Sent140 are similar, with the exception that \fedavg with personalization is nearly identical to \rbfedavg with personalization. 
 
\section{Conclusion}

We presented a robust aggregation approach, based on the geometric median and the smoothed Weiszfeld algorithm to efficiently compute it, to make federated learning more robust to settings where a fraction of the devices may be sending corrupted updates to the orchestrating server. 
The robust aggregation oracle preserves the privacy of participating devices, operating with calls to secure multi-party computation primitives enjoying privacy preservation theoretical guarantees.
\rbfedavg is available in several variants, including a fast one with a single step of robust aggregation and a one adjusting to heterogeneity with on-device personalization. All variants are readily scalable while preserving privacy, building off secure multi-party computation primitives already used at planetary scale.
The theoretical analysis of \rbfedavg  with personalization is an interesting venue for future work. The further analysis of robustness under heterogeneity is also an interesting venue for future work. 
\vspace*{0.2em}
\subsubsection*{Acknowledgments}{}
{ \small
The authors would like to thank Zachary Garrett, Peter Kairouz, Jakub Kone{\v{c}}n{\'{y}}, Brendan McMahan, 
Krzysztof Ostrowski and Keith Rush
 for fruitful discussions, as well as help with the implementation of RFA on
 Tensorflow Federated. 
This work was first presented at the Workshop on Federated Learning and Analytics in June 2019. 
This work was supported by NSF CCF-1740551, NSF CCF-1703574, NSF DMS-1839371, 
the Washington Research Foundation for innovation in Data-intensive Discovery, 
the program ``Learning in Machines and Brains'', faculty research awards,
and a JP Morgan PhD Fellowship.
}

\bibliography{bib/federated}
\bibliographystyle{abbrvnat}

\clearpage

\begin{titlepage}
   \vspace*{\stretch{1.0}}
   \begin{center}
      \Large\textbf{Supplementary Material: \\ 
      	Robust Aggregation for Federated Learning}
   \end{center}
   \vspace*{\stretch{2.0}}

   \section*{Table of Contents}
   \startcontents[sections]
   \printcontents[sections]{l}{1}{\setcounter{tocdepth}{2}}

\end{titlepage}
\appendix %

\begingroup
\let\clearpage\relax 
\onecolumn 
\endgroup

\section{Table of Notation}

We summarize the notation used throughout the paper in Table~\ref{tab:rfa:notation}. 

\begin{table}[h!]
\caption{Summary of notation.}
\label{tab:rfa:notation}
\begin{center}
\begin{adjustbox}{totalheight=\textheight-8\baselineskip}
\begin{tabular}{p{3cm}p{2cm}p{10cm}}

\textbf{Context} & \textbf{Symbol} & \textbf{Meaning} \\
\toprule

\multirow{10}{*}{Setup} &
$n$ & Total number of devices \\
&$\alpha_i$ & The weight of device $i$ \\
&$D_i$ & Data distribution of device $i$ \\
&$\mathcal{D}$ & Family of probability distributions such that $D_i \in \mathcal{D}$ for each non-corrupted device \\
&$\mathrm{width}(\mathcal{D})$ & Degree of heterogeneity in $\mathcal{D}$ \\
&$z$ & Random variable denoting the data $z \sim D_i$. 
		For example, $z=(x,y)$ is an input-output pair for supervised learning \\
&$w$ & Model parameters in $\reals^d$ \\
&$f(w;z)$ & Loss of model $w$ on example $z$ \\
&$F(w)$ & Average objective across all devices; defined in Eq.~\eqref{eq:fl:main} \\
&$w^\star$ & Optimal model parameters $w^\star = \argmin_{w \in \reals^d} F(w)$  \\
\midrule

\multirow{8}{*}{\begin{tabular}{l}FL \\algorithms \end{tabular}} &
$m$ & Number of devices chosen per round for federated learning \\
&$t$ &  index of outer iteration of \rbfedavg or \fedavg \\
&$S_t$ & Random subset of $m$ devices chosen from $\{1, \cdots, n\}$ in round $t$ \\
&$w\pow{t}$ & Global model in round $t$ \\
&$w_{i, k}\pow{t}$ & Local updates on device $i$ in round $t$ for $k=1,\cdots,\tau$ \\
&$w\pow{t+1}_i$ & Updated parameter returned by a selected client $i \in S_t$ in round $t$ \\
\midrule

\multirow{4}{*}{\begin{tabular}{l}Corruption \\ models\end{tabular}} &
$\mathcal{C}$ & Subset of clients which send corrupted updates; $\mathcal{C} \subseteq [n]$ \\
&$\rho$ & Corruption level, defined as the fraction $\sum_{i \in \mathcal{C}} \alpha_i / \sum_{i \in [n]} \alpha_i$ \\
&$\tilde D_i$ & Distribution on device $i$ due to static data poisoning, different from the original $D_i$ \\
\midrule

\multirow{6}{*}{\begin{tabular}{l}Geometric \\ median \\ definition \end{tabular}} &
$v$ & Vector in $\reals^d$; used to denote the current estimates of the geometric median \\
&$g$ & Geometric median (GM) objective, whose minimizer is the GM \\
&$\eps$ & Approximation tolerance of the GM \\
&$\nu$ & Smoothing parameter for the geometric median objective \\
&$\eta$ & Auxiliary variables used to define a surrogate $G$ \\
&$G$ & Surrogate to the geometric median objective using auxiliary variables $\eta$ \\
\midrule

\multirow{22}{*}{\begin{tabular}{l} Convergence \\ of RFA \end{tabular}} 
&$x$ & Input component of $z$ to make a prediction; input on device $i$ is denoted $x_i$ \\
&$y$ & Output component of $z$ which is the target prediction; output on device $i$ is denoted $y_i$ \\ 
&$\phi(x)$ & $d$-dimensional feature map (i.e., basis) used for linear model $\phi(x)\T w$ \\
&$\ell$ & Loss function, so that $f(w;\xi) = \ell(y, \phi(x)\T w)$, where $\xi = (x, y)$ is the data. 
		We take $\ell$ to the least-squares loss for the theory \\
&$F_i(w)$ & Local objective on device $i$ \\
&$w_i^\star$ & Local optimum of device $i$, i.e., $w_i^\star = \argmin_w F_i(w)$ \\
&$D_{X, i}$ & Marginal distribution of $D_i$ over the $x$-component of the data \\
&$D_{Y|X, i}$ & Conditional distribution of $y$ given $x$ on device $i$; $D_i = D_{X, i} \otimes D_{Y|X, i}$ \\
&$R$ & Bound on the norm of the feature map $R \ge \norm{\phi(x)}$ \\
&$L$ & Smoothness of each local objective $F_i$ \\
&$\mu$ & Strong convexity of each local objective $F_i$ \\
&$\kappa$ & Condition number $\kappa = R^2/\mu$ \\
&$\sigma^2$ & Noise variance in the linear model \\
&$H_i$ & Hessian $\grad^2 F_i(w)$ on each device $i$; note that it is constant for all $w$ \\
&$H$ & Hessian $\grad^2 F(w)$ of the average global objective $F$; we have, $H_k = (1/n)\sum_{i=1}^K \grad^2 F_i(w)$ \\
&$\niidterm_X$ & Degree of heterogeneity in the marginal distributions $D_{X, k}$ over $x$; cf. Eq.~\eqref{eq:rfa:niid_xy} \\
&$\niidterm_{Y|X}$ & Degree of heterogeneity in the condition distribution $D_{Y|X, k}$ of $y|x$; cf. Eq.~\eqref{eq:rfa:niid_xy} \\
&$\niidterm$ & Shorthand of $\mathrm{width}(\mathcal{D}_k)$, which denotes the degree of heterogeneity in $D_k$;
		defined as $\niidterm = \niidterm_X \niidterm_{Y|X}$ \\
&$\gamma$ & Learning rate of SGD \\
&$\tau_t$ & Number of local steps of SGD on each device in outer FL round $t$ \\
&$\delta$ & Confidence parameter in $(0, 1)$ \\
&$T$ & Number of rounds of federated learning \\
\midrule

Personalization &
$u_i$ & Personalization parameter of device $i$; it is a vector in $\reals^d$  \\

\bottomrule

\end{tabular}
\end{adjustbox}
\end{center}
\end{table}

\clearpage 
\section{Template Implementation of RFA in TensorFlow Federated} \label{sec:a:software}

\new{
We provide here a template implementation of \rbfedavg{} 
in Tensorflow Federated. The open source software is publicly available~\cite{rfa_tff}.}

\begin{lstlisting}[caption=Template implementation of \rbfedavg{} in Tensorflow Federated]
# Code for dataset setup, model setup, etc. comes here
federated_train_data = ...
model_fn = ... # See e.g., TFF tutorials

# Running FedAvg in TFF
import tensorflow_federated as tff
iterative_process = tff.learning.build_federated_averaging_process(model_fn)
state = iterative_process.initialize()
for round_num in range(1, num_rounds):
    state, metrics = iterative_process.next(state, federated_train_data)


# Running RFA
from federated_research.robust_aggregation import build_robust_federated_aggregation_process
iterative_process = build_robust_federated_aggregation_process(model_fn)
# Rest of the code remains unchanged
state = iterative_process.initialize()
for round_num in range(1, num_rounds):
    state, metrics = iterative_process.next(state, federated_train_data)

\end{lstlisting} 

\section{The Smoothed Weiszfeld Algorithm: Convergence Analysis} \label{sec:a:weiszfeld}

In this section, we prove the rate of the smoothed Weiszfeld algorithm in \Cref{prop:weiszfeld:main}. 
We start by a setup, prove a number of interesting properties, and finally prove \Cref{prop:weiszfeld:main} in \Cref{sec:a:weisfeld:rate}.

\subsection{Setup}
We are given distinct points $w_1, \cdots, w_m \in \reals^d$ 
and scalars $\alpha_1, \cdots, \alpha_m > 0$ such that $\sum_{i=1}^m \alpha_i = 1$. 
We make the following non-degenerateness assumption, 
which is assumed to hold throughout this work.
It is reasonable in the federated learning setting we consider.
\begin{assumption} \label{asmp:weiszfeld:noncollinear}
	The points $w_1, \cdots, w_i$ are not collinear.
\end{assumption}

\noindent
The geometric median is defined as any minimizer of
\begin{align} \label{eq:weiszfeld:fn}
	g(z) := \sum_{i=1}^m \alpha_i \norm{z - w_i} \,.
\end{align}
Under Assumption~\ref{asmp:weiszfeld:noncollinear}, $g$ is known to have a unique minimizer - we denote it by $z^\star$.

Given a smoothing parameter $\nu > 0$, its smoothed variant $g_\nu$ is 
\begin{align} \label{eq:weiszfeld:fn:smooth}
	g_\nu(z) := \sum_{i=1}^m \alpha_i \norma{(\nu)}{z - w_i} \,,
\end{align}
where
\begin{align} \label{eq:weiszfeld:norm:smooth}
	\norma{{(\nu)}}{z} := \max_{u\T u \le 1} \left\{ u\T z - \tfrac{\nu}{2}u\T u \right\} + \tfrac{\nu}{2}
	 	= \begin{cases}
	 		\frac{1}{2\nu} \normsq{z} + \frac{\nu}{2} \,, & \norm{z} \le \nu \\
	 		\norm{z} \,, & \norm{z} > \nu 
	 	\end{cases} \,.
\end{align}
In case $\nu = 0$, we define $g_0 \equiv g$.
It is known~\cite{beck2012smoothing} that $\norma{(\nu)}{\cdot}$ is $(1/\nu)$-smooth 
and that 
\begin{align} \label{eq:weiszfeld:norm:smooth:bound}
	0 \le \norma{(\nu)}{{\cdot}} - \norm{\cdot} \le \nu /2 
\end{align}
Under Assumption~\ref{asmp:weiszfeld:noncollinear}, $g_\nu$ has a unique minimizer
as well, denoted by $v_\nu^\star$.
We call $v_\nu^\star$ as the $\nu$-smoothed geometric median.

We let $B$ denote the diameter of the convex hull of $\{w_1, \cdots, w_m\}$, i.e., 
\begin{align}
	B := \diam(\conv\{w_1, \cdots, w_m\}) = \max_{z, z' \in \conv\{w_1, \cdots, w_m\}} \norm{z - z'} \,.
\end{align}
We also assume that $\nu < B$, since for all $\nu \ge B$, the function $g_\nu$ is simply a quadratic for all 
$z \in \conv\{w_1, \cdots, w_m\}$.

\subsection{Weiszfeld's Algorithm: Review}
The Weiszfeld algorithm~\cite{weiszfeld1937point} performs the iterations
\begin{align} \label{eq:weiszfeld:ns}
	v\pow{r+1} = 
	\begin{cases}
		\left( {\sum_{i=1}^m \beta_i\pow{r} w_i } \right) / \left({ \sum_{i=1}^m  \beta_i\pow{r}} \right)
				\,, & \text{ if } v\pow{r} \notin \{w_1, \cdots, w_i\} \,, \\
		w_i\,, & \text{ if } v\pow{r} = w_i \text{ for some } k \,,
	\end{cases}
\end{align}
where $\beta_i\pow{r} = {\alpha_i}/{\norm{v\pow{r} - w_i }}$.
It was shown in~\cite[Thm. 3.4]{kuhn1973note} that the sequence $\left(v\pow{r}\right)_{t=0}^\infty$ converges to the minimizer of 
$g$ from \eqref{eq:weiszfeld:fn}, provided no iterate coincides with one of the $w_i$'s. 
We modify Weiszfeld's algorithm to find the smoothed geometric median by considering
\begin{align} \label{eq:weiszfeld:smooth}
	v\pow{r+1} = \frac{ \sum_{i=1}^m \beta_i\pow{r} w_i }{ \sum_{i=1}^m  \beta_i\pow{r}} \,, \quad \text{where,} \quad
	\beta_i\pow{r} = \frac{\alpha_i}{\max\left\{\nu, \norm{v\pow{r} - w_i} \right\} } \,.
\end{align}
This is also stated in \Cref{algo:weiszfeld:smooth}.
Since each iteration of Weiszfeld's algorithm or its smoothed variant consists in taking a weighted average of the 
$w_i$'s, the time complexity is $\bigO(md)$ floating point operations per iteration.

\begin{algorithm}[tb]
	\caption{The Smoothed Weiszfeld Algorithm} 
	\label{algo:weiszfeld:smooth}
\begin{algorithmic}[1]
		\Require %
			$w_1, \cdots, w_m \in \reals^d$, 
			$\alpha_1, \cdots, \alpha_m > 0$ with $\sum_{i=1}^m \alpha_i = 1$,
			$\nu > 0$,
			number of iterations $R$,
			$v\pow{0} \in \conv\{w_1, \cdots, w_m\}$.
		\For{$r=0,1, \cdots, R-1$}
			\State Set $\eta_i\pow{r} = \max\left\{\nu, \norm{v\pow{r} - w_i} \right\}$ 
				and
				$\beta_i\pow{r} = {\alpha_i} / {\eta_i\pow{r}}$
					for $i = 1, \cdots, m$.
			\State Set 
				$
				v\pow{r+1} = \left( { \sum_{i=1}^m \beta_i\pow{r} w_i }\right) / \left({ \sum_{i=1}^m  \beta_i\pow{r}} \right)
				$.
		\EndFor
		\Ensure $v\pow{r}$.
	   	
\end{algorithmic}
\end{algorithm}

\subsection{Derivation}
We now derive Weiszfeld's algorithm with smoothing as 
as an alternating minimization algorithm 
or as an iterative minimization of a majorizing objective.

\myparagraph{Surrogate Definition}
Consider $\eta = (\eta_1, \cdots, \eta_m) \in \reals^m$ and define
$G: \reals^d \times \reals^m_{++} \to \reals$ as 
\begin{align} \label{eq:weiszfeld:joint-fn}
	G(z, \eta) = \frac{1}{2} \sum_{i=1}^m \alpha_i \left( \frac{\normsq{z - w_i}}{\eta_i} + \eta_i \right) \,.
\end{align}
Note firstly that $G$ is jointly convex in $z, \eta$ over its domain.

The first claim shows how to recover $g$ and $g_\nu$ from $G$.
\begin{claim} \label{claim:weiszfeld:surrogate}
	Consider $g, g_\nu$ and $G$ defined in Equations \eqref{eq:weiszfeld:fn}, \eqref{eq:weiszfeld:fn:smooth} and
	\eqref{eq:weiszfeld:joint-fn}, and fix $\nu > 0$. Then we have the following:
	\begin{align}
		\label{eq:weiszfeld:surrogate:nonsmooth}
		g(z) &= \inf_{\eta_1,\cdots, \eta_i > 0} G(z, \eta)\,, 
		\quad \text{and, } \\
		\label{eq:weiszfeld:surrogate:smooth}
		g_\nu(z) &= \min_{\eta_1,\cdots, \eta_i \ge \nu} G(z, \eta) \,.
	\end{align}
\end{claim}
\begin{proof}
	Define $G_i: \reals^d \times \reals_{++} \to \reals$ by
	\[
		G_i(z, \eta_i) := \frac{1}{2}  \left( \frac{\normsq{z - w_i}}{\eta_i} + \eta_i \right)\,,
	\]
	so that $G(z, \eta) = \sum_{i=1}^m \alpha_i G_i(z, \eta_i)$.

	Since $\eta_i > 0$, the arithmetic-geometric mean inequality implies that 
	$G_i(z, \eta_i) \ge \norm{z - w_i}$
	for each $i$. When $\norm{z - w_i} > 0$, the inequality above holds with equality
	when $\normsq{z - w_i}/{\eta_i} =  \eta_i$, or equivalently, $\eta_i =  \norm{z - w_i}$.
	On the other hand, when $\norm{z- w_i}  = 0$, let $\eta_i \to 0$ to conclude that 
	\[
		\inf_{\eta_i > 0} G_i(z, \eta_i) = \norm{z - w_i} \,.
	\]

	For the second part, we note that if $\norm{z-w_i} \ge \nu$, then $\eta_i =  \norm{z - w_i} \ge \nu$
	minimizes $G_i(z, \eta_i)$, so that $\min_{\eta_i \ge \nu} G_i(z, \eta_i) = \norm{z-w_i}$. 
	On the other hand, when $\norm{z-w_i} < \nu$, we note that $G_i(z, \cdot)$ is minimized
	over $[\nu, \infty)$ at $\eta_i = \nu$,
	in which case we get $G_i(z, \eta)= \normsq{z - w_i}/(2\nu) + \nu / 2$. 
	From \eqref{eq:weiszfeld:norm:smooth}, we conclude that 
	\[
		\min_{\eta_i \ge \nu} G_i(z, \eta_i) = \norma{(\nu)}{z - w_i} \,.
	\]
	The proof is complete since $G(z, \eta) = \sum_{i=1}^m \alpha_i G_i(z, \eta_i)$.
\end{proof}

\Cref{claim:weiszfeld:surrogate} now allows us to consider 
the following problem in lieu of minimizing $g_\nu$ from \eqref{eq:weiszfeld:fn:smooth}.

\begin{align} \label{eq:weiszfeld:smooth:surrogate_prob}
	\min_{
	\scriptsize{\begin{matrix} z \in \reals^d, \\ \eta_1, \cdots, \eta_m \ge \nu \end{matrix} }
	} G(z, \eta) \,.
\end{align}

\myparagraph{Alternating Minimization}
Next, we consider an alternating minimization algorithm to minimize $G$ in $z, \eta$. 
The classical technique of alternating minimization method, known also as the block-coordinate or block-decomposition method 
\citep[see, e.g.,][]{bertsekas1999nonlinear},
minimizes a function $f: X \times Y \to \reals$ using the updates
\[
	x\pow{r+1} = \argmin_{x \in X} f(x, y\pow{r}) \quad \text{ and, } y\pow{r+1} = \argmin_{y \in Y} f(x\pow{r+1}, y) \,.
\]
Application of this method to Problem~\eqref{eq:weiszfeld:smooth:surrogate_prob} yields the updates
\begin{align}
\label{eq:weiszfeld:smooth:alternating}
\begin{aligned}
	\eta\pow{r} &= \argmin_{\eta_1, \cdots, \eta_m \ge \nu}  G(v\pow{r}, \eta)
		= \left( \argmin_{\eta_i \ge \nu}  \left\{ \frac{\normsq{v\pow{r} - w_i}}{\eta_i} + \eta_i \right\} \right)_{i=1}^m \,, \\
	v\pow{r+1} &= \argmin_{z \in \reals^d} G(z, \eta\pow{r})
		= \argmin_{z \in \reals^d} \sum_{i=1}^m \frac{\alpha_i}{\eta_i\pow{r}} \normsq{z - w_i} \,.
\end{aligned}
\end{align}
These updates can be written in closed form as 
\begin{align} \label{eq:weiszfeld:smooth:closed_form}
\begin{aligned}
	\eta_i\pow{r} &= \max\{ \nu, \norm{v\pow{r} - w_i}\}\,,  \\
	v\pow{r+1} &= \left( \sum_{i=1}^m \frac{\alpha_i}{\eta_i\pow{r}} w_i \right) / 
		\left( \sum_{i=1}^m \frac{\alpha_i}{\eta_i\pow{r}} \right)  \,.
\end{aligned}
\end{align}
This gives the smoothed Weiszfeld algorithm, as pointed out by the following claim.

\begin{claim} \label{eq:weiszfeld:as:alternating:min}
	For any fixed $\nu > 0$ and starting point $v\pow{0} \in \reals^d$, the sequences $\left(v\pow{r}\right)$ produced by 
	\eqref{eq:weiszfeld:smooth} and \eqref{eq:weiszfeld:smooth:closed_form}, and hence, \eqref{eq:weiszfeld:smooth:alternating} are identical.
\end{claim}
\begin{proof}
	Follows from plugging in the expression from $\eta_i\pow{r}$ in the update for $v\pow{r+1}$ in \eqref{eq:weiszfeld:smooth:closed_form}.
\end{proof}

\myparagraph{Majorization-Minimization}
We now instantiate the smoothed Weiszfeld algorithm as a majorization-minimization scheme. 
In particular, it is the iterative minimization of a first-order surrogate in the sense of
\cite{mairal2013optimization,mairal2015incremental}.

Define $g_\nu\pow{r} : \reals^d \to \reals$ as 
\begin{align} \label{eq:weiszfeld:majorizer}
	g_\nu\pow{r}(z) := G(z, \eta\pow{r}) \,,
\end{align}
where $\eta\pow{r}$ is as defined in \eqref{eq:weiszfeld:smooth:alternating}.
The $z$-step of \eqref{eq:weiszfeld:smooth:alternating} simply sets $v\pow{r+1}$ to be the minimizer of 
$g_\nu\pow{r}$.

We note the following properties of $g_\nu\pow{r}$.
\begin{claim} \label{claim:weiszfeld:maj-min}
	For $g_\nu\pow{r}$ defined in \eqref{eq:weiszfeld:majorizer}, the following properties hold:
	\begin{align}
		\label{eq:weiszfeld:surrogate:majorizing}
		g_\nu\pow{r}(z) &\ge g_\nu(z)\,, \quad \text{for all} z \in \reals^d \,, \\
		\label{eq:weiszfeld:surrogate:0th-order}
		g_\nu\pow{r}(v\pow{r}) &= g_\nu(v\pow{r})\,, \quad \text{and,} \\
		\label{eq:weiszfeld:surrogate:1st-order}
		\grad g_\nu\pow{r}(v\pow{r}) &= \grad g_\nu(v\pow{r}) \,.
	\end{align}
	Moreover $g\pow{r}$ can also be written as 
	\begin{align} \label{eq:weisfeld:surrogate:taylor}
		g_\nu\pow{r}(z) = g_\nu(v\pow{r}) + \grad g_\nu(v\pow{r})\T \left(z - v\pow{r}\right) + \frac{L\pow{r}}{2} \normsq{z - v\pow{r}}\,,
	\end{align}
	where 
	\begin{align} \label{eq:weiszfed:l}
		L\pow{r} := \sum_{i=1}^m \frac{\alpha_i}{\eta_i\pow{r}} \,.
	\end{align}
\end{claim}
\begin{proof}
	The first part follows because
	\[
		g_\nu(z) = \min_{\eta_1, \cdots, \eta_m} G(z, \eta) \le G(z, \eta\pow{r}) = g_\nu\pow{r}(z) \,.
	\]
	For Eq.~\eqref{eq:weiszfeld:surrogate:0th-order}, note that the inequality above is an equality at $v\pow{r}$
	by the definition of $\eta\pow{r}$ from \eqref{eq:weiszfeld:smooth:alternating}.
	To see \eqref{eq:weiszfeld:surrogate:1st-order}, note that
	\[
		\grad g_\nu(z) = \sum_{i=1}^m \alpha_m \frac{z - w_i}{\max\{\nu, \norm{z - w_i}\}} \,.
	\]
	Then, by the definition of $\eta\pow{r}$ from \eqref{eq:weiszfeld:smooth:closed_form}, we get that 
	\[
		\grad g_\nu(v\pow{r}) = \sum_{i=1}^m \frac{\alpha_m}{\eta_i\pow{r}} (v\pow{r} - w_i) = \grad g_\nu\pow{r}(v\pow{r}) \,.
	\]
	The obtain the expansion \eqref{eq:weisfeld:surrogate:taylor}, we write out the Taylor expansion of 
	the quadratic $g\pow{r}(z)$ around $v\pow{r}$ to get 
	\[
		g_\nu\pow{r}(z) = g_\nu\pow{r}(v\pow{r}) + \grad g_\nu\pow{r}(v\pow{r})\T \left(z - v\pow{r}\right) 
				+ \frac{L\pow{r}}{2} \normsq{z - v\pow{r}}\,,
	\]
	and complete the proof by plugging in \eqref{eq:weiszfeld:surrogate:0th-order} and \eqref{eq:weiszfeld:surrogate:1st-order}.
\end{proof}

\myparagraph{Gradient Descent}
The next claim rewrites the smoothed Weiszfeld algorithm as gradient descent on $g_\nu$.

\begin{claim} \label{claim:weiszfeld:gradient-descent}
	Equation~\eqref{eq:weiszfeld:smooth} can also be written as 
	\begin{align} \label{eq:weiszfeld:gradient-descent}
		v\pow{r+1} = v\pow{r} - \frac{1}{L\pow{r}} \grad g_\nu(v\pow{r})\,,
	\end{align}
	where $L\pow{r}$ is as defined in \eqref{eq:weiszfed:l}.
\end{claim}
\begin{proof}
	Use $v\pow{r+1} = \argmin_{z \in \reals^d} g_\nu\pow{r}(z)$, 
	where $g_\nu\pow{r}$ is written using \eqref{eq:weisfeld:surrogate:taylor}.
\end{proof}

\subsection{Properties of Iterates}
The first claim reasons about the iterates $v\pow{r}, \eta\pow{r}$.

\begin{claim} \label{eq:weiszfeld_alg:prop}
	Starting from any $v\pow{0} \in \conv\{w_1, \cdots, w_m\}$, the
	sequences $(\eta\pow{r})$ and $(v\pow{r})$ produced by %
	\Cref{algo:weiszfeld:smooth}
	satisfy 
	\begin{itemize}
		\item $v\pow{r} \in \conv\{w_1, \cdots, w_m\}$ for all $t \ge 0$, and, 
		\item $\nu \le \eta_i\pow{r} \le B\,$ for all $i=1, \cdots, m$, and $t \ge 1$, 
	\end{itemize}
	where $B = \diam(\conv\{w_1, \cdots, w_m\})$. Furthermore, $L\pow{r}$ 
	defined in \eqref{eq:weiszfed:l} satisfies $1/B \le L\pow{r} \le 1/\nu$ for all $t\ge 0$.
\end{claim}
\begin{proof}
	The first part follows for $t \ge 1$ from the update \eqref{eq:weiszfeld:smooth}, where
	Claim~\ref{eq:weiszfeld:as:alternating:min} shows the equivalence of 
	\eqref{eq:weiszfeld:smooth} and \eqref{eq:weiszfeld:smooth:alternating}. 
	Then case of $t=0$ is assumed.
	The second part follows from \eqref{eq:weiszfeld:smooth:closed_form}
	and the first part. The bound on $L\pow{r}$ follows from the second part
	since $\sum_{i=1}^m \alpha_i = 1$.
\end{proof}

The next result shows that it is a descent algorithm. Note that the non-increasing nature of the  
sequence $\left( g_\nu(v\pow{r}) \right)$ also follows from the majorization-minimization 
viewpoint~\citep{mairal2015incremental}.
Here, we show that this sequence is strictly decreasing.
Recall that $v_\nu^\star$ 
is the unique minimizer of $g_\nu$. 
\begin{lemma} \label{lemma:weiszfeld:monotonicity}
	The sequence $(v\pow{r})$ produced by \Cref{algo:weiszfeld:smooth} satisfies
	$g_\nu(v\pow{r+1}) < g_\nu(v\pow{r})$ unless $v\pow{r} = v_\nu^\star$.
\end{lemma}
\begin{proof}
	Let $\mathcal{E}_\nu = \{\eta \in \reals^m \,:\, \eta_i \ge \nu \text{ for } k=1,\cdots, m \}$.
	Starting with \eqref{eq:weiszfeld:surrogate:smooth}, we successively deduce, 
	\begin{align*}
		g_\nu(v\pow{r+1}) 
		&= \min_{\eta_1, \cdots, \eta_m \ge \nu} G(v\pow{r+1}, \nu) \\
		&\le G(v\pow{r+1}, \eta\pow{r}) \\
		&= \min_{z \in \reals^d} G(z, \eta\pow{r}) \\
		&\le G(v\pow{r}, \eta\pow{r}) \\
		&= \min_{\eta_1, \cdots, \eta_m \ge \nu} G(v\pow{r}, \eta) \\
		&= g_\nu(v\pow{r}) \,.
	\end{align*}
	Here, we used the fact that $v\pow{r+1}$ minimizes $G(\cdot, \eta\pow{r})$ over $\reals^d$ and that
	$\eta\pow{r}$ minimizes $G(v\pow{r}, \cdot)$ over $\mathcal{E}_\nu$.

	Suppose now that $g_\nu(v\pow{r+1}) = g_\nu(v\pow{r})$.
	In this case, both the inequalities above hold with equality.
	Since $G(\cdot, \eta\pow{r})$ is $L\pow{r}$-strongly convex 
	where $L\pow{r} \ge 1/B$ (cf. Claim~\ref{eq:weiszfeld_alg:prop}),
	this implies that $v\pow{r} = \argmin_{z \in \reals^d} G(z, \eta\pow{r})$.
	By definition then, $\eta\pow{r+1} = \eta\pow{r}$ is the unique minimizer of $G(v\pow{r}, \cdot)$
	over $S$, since $G(v\pow{r}, \cdot)$ is strictly convex. 
	The associated first-order optimality conditions are the following:
	\[
		\grad_z G(v\pow{r}, \eta\pow{r}) = 0\,, 
		\quad \text{and,} \quad
		\grad_\eta G(v\pow{r}, \eta\pow{r})\T(\eta - \eta\pow{r}) \ge 0 \quad \forall \eta \in \mathcal{E}_\nu \,.
	\]
	Putting these together, we find that 
	the pair $(v\pow{r}, \eta\pow{r})$ satisfies the first-order optimality conditions for 
	$G$ over the domain $\reals^d \times \mathcal{E}_\nu$. Hence, $v\pow{r} = v_\nu^\star$.
\end{proof}

The next lemma shows that $\norm{v\pow{r} - z^\star}$ is non-increasing.
This property was shown in \cite[Corollary 5.1]{beck15weiszfeld} for the case of Weiszfeld algorithm without smoothing.

\begin{lemma} \label{lemma:weiszfeld:contraction}
	The sequence $(v\pow{r})$ produced by \Cref{algo:weiszfeld:smooth} satisfies for all $t \ge 0$,
	\[
		\norm{v\pow{r+1} - v_\nu^\star} \le \norm{v\pow{r} - v_\nu^\star} \,.
	\]
	Furthermore, if $g_\nu(v\pow{r+1}) \ge g_\nu(z^\star)$, then
	it holds that 
	\[
		\norm{v\pow{r+1} - z^\star} \le \norm{v\pow{r} - z^\star} \,.
	\]
\end{lemma}
\begin{proof}
	First note from Claim~\ref{claim:weiszfeld:gradient-descent} that 
	\begin{align} \label{eq:weiszfeld:pf:contraction:1}
		\grad g_\nu(v\pow{r}) = L\pow{r}(v\pow{r} - v\pow{r+1}) \,,
	\end{align}
	where $L\pow{r}$ is defined in \eqref{eq:weiszfed:l}.
	Starting from the results of Claim~\ref{claim:weiszfeld:maj-min}, we observe for any $z$ that,
	\begin{align*}
		g_\nu(v\pow{r+1})
		&\stackrel{\eqref{eq:weiszfeld:surrogate:majorizing}}{\le} g_\nu\pow{r}(v\pow{r+1}) \\
		&\stackrel{\eqref{eq:weisfeld:surrogate:taylor}}{=} 
			g_\nu(v\pow{r}) + \grad g_\nu(v\pow{r})\T \left(v\pow{r+1} - v\pow{r}\right) 
				+ \frac{L\pow{r}}{2} \normsq{v\pow{r+1} - v\pow{r}} \\
		&\stackrel{(*)}{\le} g_\nu(z) + \grad g_\nu(v\pow{r})\T \left(v\pow{r+1} - z\right) 
				+ \frac{L\pow{r}}{2} \normsq{v\pow{r+1} - v\pow{r}} \\
		&\stackrel{\eqref{eq:weiszfeld:pf:contraction:1}}{=} g_\nu(z) 
				+ L\pow{r} \left( v\pow{r} - v\pow{r+1} \right)\T \left(v\pow{r+1} - z\right) 
				+ \frac{L\pow{r}}{2} \normsq{v\pow{r+1} - v\pow{r}} \,,
	\end{align*}
	where $(*)$ following from the convexity of $g_\nu$ as 
	$g_\nu(z) \ge g_\nu(v\pow{r}) + \grad g_\nu(v\pow{r})\T(z - v\pow{r})$.
	Next, we use the Pythagorean identity: for any $a, b, c \in \reals^d$, it holds that
	\[
		\normsq{b-a} + 2 (b-a)\T (a - c) = \normsq{b- c} - \normsq{a-c} \,.
	\]
	With this, we get, 
	\begin{align*}
		g_\nu(v\pow{r+1}) \le g_\nu(z) + \frac{L\pow{r}}{2} \left( \normsq{v\pow{r} - z} - \normsq{v\pow{r+1} - z} \right) \,.
	\end{align*}
	Plugging in $z = v_\nu^\star$, the fact that $g_\nu(v\pow{r+1}) \ge g_\nu(z^\star)$ implies that 
	$\normsq{v\pow{r+1} - v_\nu^\star} \le \normsq{v\pow{r} - v_\nu^\star}$, since $L\pow{r} \ge 1/B$ is 
	strictly positive.
	Likewise, for $z = z^\star$, the claim holds under the condition that $g_\nu(v\pow{r+1}) \ge g_\nu(z^\star)$.
\end{proof}

\subsection{Rate of Convergence} \label{sec:a:weisfeld:rate}
We are now ready to prove the global sublinear rate of convergence of \Cref{algo:weiszfeld:smooth}.

\begin{theorem} \label{eq:weiszfeld:main}
	The iterate 
	$v\pow{R}$ produced by
	\Cref{algo:weiszfeld:smooth}
	with input $v\pow{0} \in \conv\{w_1, \cdots, w_m\}$ and $\nu > 0$ satisfies
	\[
		g_\nu(v\pow{R}) - g_\nu(z^\star_\nu)  \le \frac{2 \normsq{v\pow{0} - z^\star_\nu}}{\sum_{s=0}^{R-1} 1/ L\pow{s}}
			\le \frac{2 \normsq{v\pow{0} - z^\star_\nu}}{\widehat \nu R} \,,
	\]
	where $L\pow{s} = \sum_{i=1}^m {\alpha_i} / {\eta_i\pow{s}}$ is defined in \eqref{eq:weiszfed:l}, and 
	\begin{align} \label{eq:thm:weiszfeld:effective-nu}
		\widehat \nu = \adjustlimits \min_{s = 0, \cdots, R-1} \min_{i\in[m]} \max\{\nu, \norm{v\pow{s} - w_i}\} \ge \nu \,.
	\end{align}
	Furthermore, it holds that 
	\[
		g(v\pow{R}) - g(z^\star) 
			\le \frac{2 \normsq{v\pow{0} - z^\star}}{\sum_{s=0}^{R-1} 1/ L\pow{s}} + \frac{\nu}{2}
			\le \frac{2 \normsq{v\pow{0} - z^\star}}{\widehat \nu R} + \frac{\nu}{2} \,.
	\]
\end{theorem}
\begin{proof}
	With the descent and contraction properties of \Cref{lemma:weiszfeld:monotonicity}
	and \Cref{lemma:weiszfeld:contraction} respectively, the proof now follows the classical 
	proof technique of gradient descent \citep[e.g.,][Theorem 2.1.13]{nesterov2013introductory}.
	Starting from the results of Claim~\ref{claim:weiszfeld:maj-min}, we observe for any $z$ that,
	\begin{align}
		\nonumber
		g_\nu(v\pow{r+1})
		&\stackrel{\eqref{eq:weiszfeld:surrogate:majorizing}}{\le} g_\nu\pow{r}(v\pow{r+1}) \\
		\nonumber
		&\stackrel{\eqref{eq:weisfeld:surrogate:taylor}}{=} 
			g_\nu(v\pow{r}) + \grad g_\nu(v\pow{r})\T \left(v\pow{r+1} - v\pow{r}\right) 
				+ \frac{L\pow{r}}{2} \normsq{v\pow{r+1} - v\pow{r}} \\
		\label{eq:weiszfeld:proof:1}
		&\stackrel{\eqref{eq:weiszfeld:gradient-descent}}{=} 
			g_\nu(v\pow{r}) - \frac{1}{2 L \pow{r}} \normsq{\grad g_\nu (v\pow{r})} \,.
	\end{align}
	\myparagraph{Convergence on $g_\nu$}
	For ease of notation, we let $\widetilde \Delta_r := g_\nu(v\pow{r}) - g_\nu(v_\nu^\star)$.
	We assume now that $\widetilde \Delta_{r+1}$ is nonzero, and hence, so is $\widetilde \Delta_r$ (Lemma~\ref{lemma:weiszfeld:monotonicity}). 
	If $\widetilde \Delta_{r+1}$ were zero, then the theorem would hold trivially at $t+1$.

	Now, from convexity of $g_\nu$ and the Cauchy-Schwartz inequality, we get that 
	\[
		\widetilde \Delta_r \le \grad g_\nu(v\pow{r})\T\left( v\pow{r} - z^\star_\nu \right) 
			\le \norm{\grad g_\nu(v\pow{r})} \norm{v\pow{r} - z^\star_\nu} \,.
	\]
	Plugging this in, we get, 
	\begin{align*}
		\widetilde \Delta_{r+1} - \widetilde \Delta_r &\le -\frac{1}{2L\pow{r}} \frac{\widetilde \Delta_r^2}{\normsq{v\pow{r} - z^\star_\nu}} \\
			&\le - \frac{1}{2 L\pow{r}} \frac{\widetilde \Delta_r^2}{\normsq{v\pow{0} - z^\star_\nu}} \,,
	\end{align*}
	where we invoked Lemma~\ref{lemma:weiszfeld:contraction}.

	Now, we divide by $\widetilde \Delta_r \widetilde \Delta_{r+1}$, which is nonzero by assumption,
	and use $\widetilde \Delta_r / \widetilde \Delta_{r+1} \ge 1$ (Lemma~\ref{lemma:weiszfeld:monotonicity}) to get
	\begin{align*}
		\frac{1}{\widetilde \Delta_r} - \frac{1}{\widetilde \Delta_{r+1}} 
			&\le - \frac{1}{2 L\pow{r}} \left( \frac{\widetilde \Delta_r}{\widetilde \Delta_{r+1}} \right) 
				\frac{1}{\normsq{v\pow{0} - z^\star_\nu}} \\
			&\le - \frac{1}{2L\pow{r} \normsq{v\pow{0} - z^\star_\nu}}  \,.
	\end{align*}
	Telescoping, we get, 
	\begin{align*}
		\frac{1}{\widetilde \Delta_{r}} \ge \frac{1}{\widetilde \Delta_{r}} - \frac{1}{\widetilde \Delta_0} 
			\ge \left( \sum_{s=0}^{r-1} \frac{1}{L\pow{s}} \right)  \frac{1}{2 \normsq{v\pow{0} - z^\star_\nu}} \,.
	\end{align*}
	This proves the first inequality to be proved. The second inequality follows 
	from the definition in Eq.~\eqref{eq:weiszfed:l} since $\sum_{i=1}^m \alpha_i = 1$.

	\myparagraph{Convergence on $g$}
	The proof follows along the same ideas as the previous proof. Define $\Delta_r := g_\nu(v\pow{r}) - g_\nu(z^\star)$.
	Suppose $\Delta_r > 0$. Then, we proceed as previously for any $s < t$ to note 
	by convexity and Cauchy-Schwartz inequality that 
	\[
		\Delta_s \le \norm{\grad g_\nu(v\pow{s})} \norm{v\pow{s} - z^\star} \,.
	\]
	Again, plugging this into \eqref{eq:weiszfeld:proof:1}, using that 
	$\Delta_s / \Delta_{s+1} \ge 1$ and invoking Lemma~\ref{lemma:weiszfeld:contraction} gives
	(since $\Delta_s > 0$)
	\[
		\frac{1}{\Delta_s} - \frac{1}{\Delta_{s+1}} \le - \frac{1}{2 L\pow{s} \normsq{v\pow{0} - z^\star}} \,.
	\]
	Telescoping and taking the reciprocal gives
	\[
		g_\nu(v\pow{R}) - g_\nu(z^\star) = \Delta_r \le \frac{2 \normsq{v\pow{0} - z^\star}}{\sum_{s=0}^{R-1} 1/ L\pow{s}} \,.
	\]
	Using~\eqref{eq:weiszfeld:norm:smooth:bound} completes the proof for the case that $\Delta_r > 0$. 
	Note that if $\Delta_r \le 0$, it holds that $\Delta_{t'} \le 0$ for all $t' > t$.
	In this case, $g_\nu(v\pow{r}) - g_\nu(z^\star) \le 0$. Again,~\eqref{eq:weiszfeld:norm:smooth:bound}
	implies that $g(v\pow{r}) - g(z^\star) \le \nu/2$, which is trivially upper bounded by the 
	quantity stated in the theorem statement. This completes the proof.
\end{proof}

\myparagraph{Faster Rate of Convergence}
We now make an additional assumption:

\begin{assumption} \label{asmp:weiszfeld:optimum}
	The geometric median $z^\star$ does not coincide with any of $w_1, \cdots, w_m$.
	In other words, 
	\begin{align} \label{eq:weiszfeld:asmp2}
		\widetilde \nu := \min_{i=1,\cdots, m} \norm{z^\star - w_i} > 0 \,.
	\end{align}
\end{assumption}

\begin{remark}
	\cite[Lemma 8.1]{beck15weiszfeld} show a lower bound on $\widetilde \nu$ in terms of 
	$\alpha_1, \cdots, \alpha_m$ and $w_1, \cdots, w_m$.
\end{remark}

Now, we analyze the condition under which the $z^\star = v_\nu^\star$.

\begin{lemma} \label{lemma:weiszfeld:no-smoothing}
	Under Assumption~\ref{asmp:weiszfeld:optimum}, 
	we have that $z^\star = v_\nu^\star$\,
	for all $\nu \le \widetilde \nu$,   
	where $\widetilde \nu$ is defined in \eqref{eq:weiszfeld:asmp2}.
\end{lemma}
\begin{proof}
	By the definition of the smooth norm in \eqref{eq:weiszfeld:norm:smooth}, 
	we observe that $\norma{(\nu)}{z^\star - w_i} = \norm{z^\star - w_i}$
	for all $\nu \le \widetilde \nu$, and hence, $g_\nu(z^\star) = g(z^\star)$.
	For any $z \in \reals^d$, we have,
	\begin{align*}
		g_\nu(z) 
		\stackrel{\eqref{eq:weiszfeld:norm:smooth:bound}}{\ge} g(z)
		\ge g(z^\star)
		= g_\nu(z^\star) \,,
	\end{align*}
	or that $z^\star = v_\nu^\star$.
\end{proof}

In this case, we get a better rate on the non-smooth objective $g$.

\begin{corollary} \label{cor:a:weiszfeld:better_rate}
	Consider the setting of Theorem~\ref{eq:weiszfeld:main}
	where Assumption~\ref{asmp:weiszfeld:optimum} holds and $\nu \le \widetilde \nu$.
	Then, the iterate $v\pow{R}$ produced by \Cref{algo:weiszfeld:smooth}
	satisfies,
	\[
		g(v\pow{r}) - g(z^\star)  \le \frac{2 \normsq{v\pow{0} - z^\star}}{\widehat \nu R} \,,
	\]
	where $\widehat \nu$ is defined in Eq.~\eqref{eq:thm:weiszfeld:effective-nu}.
\end{corollary}
\begin{proof} 
	This follows from \Cref{eq:weiszfeld:main}'s bound on $g_\nu(v\pow{R}) - g_\nu(v_\nu^\star)$ 
	with the observations that 
	$g(v\pow{R}) \stackrel{\eqref{eq:weiszfeld:norm:smooth:bound}}{\le} g_\nu(v\pow{R})$
	and $g(z^\star) = g_\nu(z^\star)$ (see the proof of \Cref{lemma:weiszfeld:no-smoothing}).
\end{proof}

The previous corollary obtains the same rate as \cite[Theorem 8.2]{beck15weiszfeld}, up to constants 
upon using the bound on $\widetilde \nu$ given by \cite[Lemma 8.1]{beck15weiszfeld}.

We also get as a corollary a bound on the performance of Weiszfeld's original algorithm without smoothing, 
although it could be numerically unstable in practice. This bound depends on the actual iterates, so it 
is not informative about the performance of the algorithm a priori.

\begin{corollary} \label{cor:a:weiszfeld:nonsmooth_rate}
	Consider the setting of \Cref{eq:weiszfeld:main}. 
	Under Assumption~\ref{asmp:weiszfeld:optimum}, 
	suppose the sequence $(v\pow{r})$ produced by Weiszfeld's algorithm in Eq.~\eqref{eq:weiszfeld:ns}
	satisfies $\norm{v\pow{r} - w_i} > 0$ for all $r$ and $i$, then it also
	satisfies 
	\[
		g(v\pow{R}) - g(z^\star)  \le \frac{2 \normsq{v\pow{0} - z^\star}}{\nu\pow{R} R} \,.
	\]
	where $\nu \pow{r}$ is given by
	\[
	\nu\pow{r} = \min\left\{\widetilde \nu, 
	\adjustlimits \min_{s = 0, \cdots, r} \min_{i\in[m]}  \norm{v\pow{s} - w_i}
	\right\} \,.
	\]
\end{corollary}
\begin{proof}
	Under these conditions, note that the sequence $(v\pow{s})_{s=0}^t$ produced by the Weiszfeld algorithm without smoothing 
	coincides with the sequence $(v_{\nu\pow{r}}\pow{s})_{s = 0}^t$ produced by the smoothed Weiszfeld algorithm
	at level $\nu = \nu \pow{r}$. Now apply \Cref{cor:a:weiszfeld:better_rate}.
\end{proof}

\subsection{Comparison to Previous Work}

We compare the results proved in the preceding section to prior work on the subject.

\myparagraph{Comparison to \cite{beck15weiszfeld}}
The authors present multiple different variants of the Weiszfeld algorithm. 
For a particular choice of initialization, they can guarantee that a rate 
of the order of $1/\widetilde \nu R$. It is not clear how this choice of initialization can be implemented using a 
secure average oracle since, if at all. This is because it requires the computation of all pairwise distances $\norm{w_i - w_{i'}}$.
Moreover, a naive implementation of their algorithm could be numerically unstable since it would involve division
by small numbers. Guarding against division by small numbers would lead to the smoothed variant considered here.
Note that our algorithmic design choices are driven by the federated learning setting.

\myparagraph{Comparison to \cite{beck15convergence}}
The author studies general alternating minimization algorithms, including the Weiszfeld algorithm 
as a special case, with a different smoothing than the one considered here. 
While their algorithm does not suffer from numerical issues arising from division by small numbers, 
it always suffers a bias from smoothing. 
On the other hand, the smoothing considered here is more natural in that it 
reduces to Weiszfeld's original algorithm
when $\norm{v\pow{r} - w_i} > \nu$, i.e., when we are not at a risk of dividing by small numbers.
Furthermore, the bound in 
\Cref{eq:weiszfeld:main} exhibits a better dependence on the initialization $v\pow{0}$.

\section{Numerical Simulations: Full Details} \label{sec:a:expt}
The section contains a full description of the experimental setup as well as additional results.

We start with the dataset and task description in \Cref{sec:a:expt:dataset},
hyperparameter choices in \Cref{sec:a:expt:hyperparam:choice},
and evaluation methodology in \Cref{sec:a:expt:eval}.
We provide some extra numerical results in \Cref{sec:a:additional-results}.

\subsection{Datasets and Task Description} \label{sec:a:expt:dataset}
We experiment with three tasks, (1) handwritten-letter recognition, (2) character-level language modeling, and, (3) sentiment analysis.
As discussed in \Cref{sec:setup:fl}, we take the weight $\alpha_i \propto N_i$, which is the number of data points available on device $i$.

\subsubsection{Handwritten-Letter Recognition}
The first dataset is the EMNIST dataset \citep{cohen2017emnist} for handwritten letter recognition.

\myparagraph{Data}
Each inpt $x$ is a gray-scale image resized to $28\times 28$. Each output $y$ 
is categorical variable which takes 62 different values, one per class of letter (0-9, a-z, A-Z).

\myparagraph{Formulation}
The task of handwritten letter recognition is cast as a multi-class classification problem with 
62 classes.

\myparagraph{Distribution of Data}
The handwritten characters in the images are annotated by the writer of the character as well.
We use a non-i.i.d. split of the data grouped by a writer of a given image. We discard
devices with less than 100 total input-output pairs (both train and test), leaving a total of 3461 devices.
Of these, we sample 1000 devices to use for our simulations, corresponding to about 
$30\%$ of the data. This selection held constant throughout the simulations.
The number of training examples across these devices summarized in the following statistics:
median 160, mean 202, standard deviation 77, maximum 418 and minimum 92.
This preprocessing was performed using LEAF \citep{caldas2018leaf}.

\myparagraph{Models}
For the model $\varphi$, we consider two options: 
a linear model and a convolutional neural network.
\begin{itemize}[nolistsep,leftmargin=1em]
	\item Linear Model: The linear model maintains parameters $w_1, \cdots, w_{62} \in \reals^{28 \times 28}$.
		For a given image $x$, class $l$ is assigned score $\inp{w_l}{x}$, which is then converted to a probability 
		using a softmax operation as 
		$p_l = \exp(\inp{w_l}{x}) / \sum_{l'} \exp(\inp{w_{l'}}{x})$. For a new input image $x$, 
		the prediction is made as $\argmax_{l} \inp{w_l}{x}$.
	\item Convolutional Neural Network (ConvNet): The ConvNet \citep{lecun1998gradient} 
		we consider contains two 
		convolutional layers with max-pooling, followed by a fully connected hidden layer, 
		and another fully connected (F.C.) layer with 62 outputs. When given an input image $x$, 
		the output of this network is assigned as the scores of each of the classes. 
		Probabilities are assigned similar to the linear model with a softmax operation on the scores. The schema 
		of network is given below:
		\begin{align*}
			&\begin{matrix}
				\text{Input} \\
				28 \times 28
			\end{matrix}
			\longrightarrow 
			\begin{matrix}
				\text{Conv2D} \\
				\text{filters} = 32  \\
				\text{kernel} = 5 \times 5
			\end{matrix}
			\longrightarrow
			\begin{matrix}
				\text{ReLU}
			\end{matrix}
			\longrightarrow
			\begin{matrix}
				\text{Max Pool} \\
				\text{kernel } = 2\times 2\\
				\text{stride} = 2
			\end{matrix}
			\longrightarrow  
			\begin{matrix}
				\text{Conv2D} \\
				\text{filters} = 64  \\
				\text{kernel} = 5 \times 5
			\end{matrix}
			\longrightarrow \\ &
			\begin{matrix}
				\text{ReLU}
			\end{matrix}
			\longrightarrow 
			\begin{matrix}
				\text{Max Pool} \\
				\text{kernel } = 2\times 2\\
				\text{stride} = 2
			\end{matrix}
			\longrightarrow
			\begin{matrix}
				\text{F.C.} \\
				\text{units} = 2048 
			\end{matrix}
			\longrightarrow
			\begin{matrix}
				\text{ReLU}
			\end{matrix}
			\longrightarrow
			\begin{matrix}
				\text{F.C.} \\
				\text{units} = 62
			\end{matrix}
			\longrightarrow
			\text{ score }
			\end{align*}
\end{itemize}

\myparagraph{Loss Function}
We use the multinomial logistic loss $\ell(y, p) = -\log p_y$, 
for probabilities $p = (p_1, \cdots, p_{62})$
and $y \in \{1, \cdots, 62\}$.
In the linear model case, it is equivalent to the classical softmax regression.

\myparagraph{Evaluation Metric}
The model is evaluated based on the classification accuracy on the test set.

\subsubsection{Character-Level Language Modeling}
The second task is to learn a character-level language model over the 
Complete Works of Shakespeare~\cite{shakespeare}. The goal is to read a few characters and predict the next character which
appears.

\myparagraph{Data}
The dataset consists of text from the Complete Works of William Shakespeare as raw text.

\myparagraph{Formulation}
We formulate the task as a multi-class classification problem with 
53 classes (a-z, A-Z, other) as follows.
At each point, we consider the previous $H=20$ characters, 
and build $x \in \{0, 1\}^{H\times 53}$ as a one-hot encoding of these $H$ characters. 
The goal is then try to predict the next character, which can belong to 53 classes.
In this manner, a text with $l$ total characters gives $l$ input-output pairs.

\myparagraph{Distribution of Data}
We use a non-i.i.d. split of the data. Each role in a given play (e.g., Brutus from The Tragedy of Julius Caesar) 
is assigned as a separate device. All devices with less than 100 total examples are discarded, leaving 628 devices.
The training set is assigned a random 90\% of the input-output pairs, and the other rest are held out for testing.
This distribution of training examples is extremely skewed, with the following statistics:
median 1170, mean 3579, standard deviation 6367, maximum 70600 and minimum 90.
This preprocessing was performed using LEAF \citep{caldas2018leaf}.

\myparagraph{Models}
We use a long-short term memory model (LSTM) \citep{hochreiter1997long} 
with $128$ hidden units for this purpose. This is followed by a fully connected layer with 53 outputs,
the output of which is used as the score for each character. 
As previously, probabilities are obtained using the softmax operation.

\myparagraph{Loss Function}
We use the multinomial logistic loss.

\myparagraph{Evaluation Metric}
The model is evaluated based on the accuracy of next-character prediction on the test set.

\subsubsection{Sentiment Analysis}
The third task is analyze the sentiment of tweets as positive or negative. 

\myparagraph{Data}
Sent140~\citep{go2009twitter} is a text dataset of 1,600,498 tweets produced by 660,120 Twitter accounts. Each tweet is represented by a character string with emojis redacted. Each tweet is labeled with a binary sentiment reaction (i.e., positive or negative), which is inferred based on the emojis in the original tweet. 

\myparagraph{Formulation}
The task is a binary classification problem, with the output being a positive or negative sentiment, while the input is the raw text of the tweet. 

\myparagraph{Distribution of Data}
We use a non-i.i.d. split of the data. 
Each client device represents a Twitter user and contains tweets from this user. We discarded all clients containing less that 50 tweets, leaving only 877 clients.
The training set is assigned a random 80\% of the input-output pairs, and the other rest are held out for testing.
This distribution of training examples across client devices is skewed, with the following statistics:
median 55, mean 65.3, standard deviation 32.4, maximum 439 and minimum 40.
This preprocessing was performed using LEAF \citep{caldas2018leaf}.

\myparagraph{Models}
We use a linear model $\varphi(x; w) = w\T \phi(x)$, where the feature representation
$\phi(x) \in \reals^{50}$ of text $x$ is obtained as the average of the GloVe embeddings~\cite{pennington2014glove} $G(\cdot)$ of each word in the tweet, i.e.,
\[
	\phi(x)  = \frac{1}{|x|} \sum_{i=1}^{|x|} G(x_i) \,.
\]

\myparagraph{Loss Function}
We use the binary logistic loss.

\myparagraph{Evaluation Metric}
We use the binary classification accuracy.

\subsection{Methods, Hyperparameters and Variants} 
\label{sec:a:expt:hyperparam:choice}
We first describe the corruption model, followed by various methods tested.

\subsubsection{Corruption Model}
Since the goal of this work to test the robustness of federated learning models in the setting of high corruption, we artificially corrupt updates while controlling the level of corruption. We use the following corruption models.

\myparagraph{Data Corruption}
	This is an example of static data poisoning.
	The model training procedure is not modified, but the data fed into the model is modified. 
	In particular, we take a modification $\tilde {D}_i$ of 
	the local dataset $D_i$ of client $i$ and run the training algorithm on this different dataset.
	The exact nature of the modification depends on the dataset:
	\begin{itemize}[nolistsep,leftmargin=1em]
		\item EMNIST: We take the negative of the image $x$. Mathematically, 
			$\tilde D_i(x, y) = D_i(1-x, y)$, assuming the pixels of $x$
			are normalized to lie in $[0, 1]$. The labels are left unmodified.
		\item Shakespeare: We reverse the original text. 
			Mathematically, 
			$\tilde D_i(c_1\cdots c_{20}, c_{21}) = D_i(c_{21}\cdots c_2, c_1)$
			This is illustrated in Fig.~\ref{fig:a:expt:shake:noise}.
			The labels are left unmodified.
		\item Sent140: We flip the label, i.e., 
			$\tilde D_i(x, y) = D_i(x, -y)$. The text in the tweet remains unchanged.
	\end{itemize}

\begin{figure*}
\centering
	\adjincludegraphics[width=0.6\linewidth,trim={80pt 320pt 100pt 90pt},clip=true]{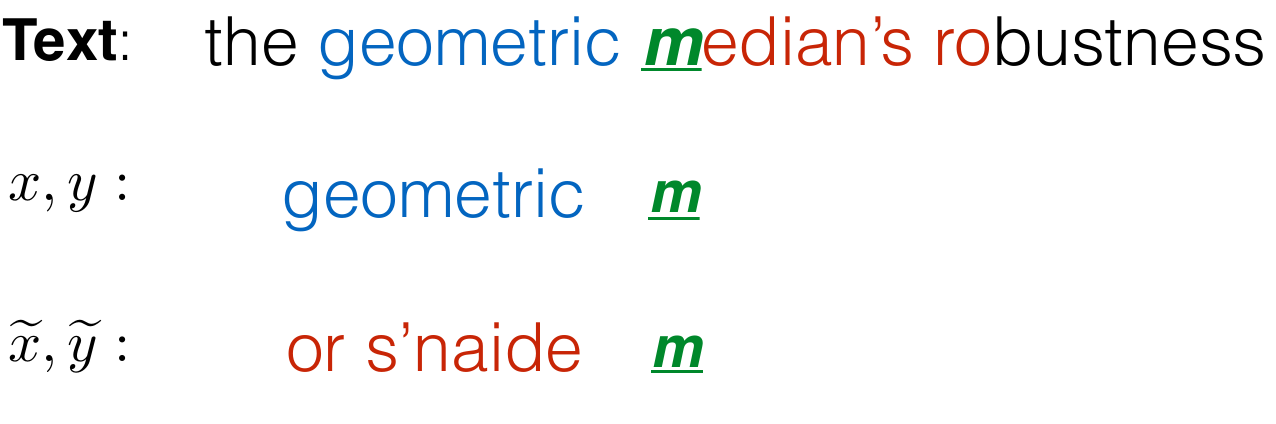}
	\caption{Illustration of the data corruption introduced in the Shakespeare dataset. The first line denotes the original 
	text. The second line shows the effective $x$ when predicting the ``m'' of the word ``median''. The second line 
	shows the corresponding $\widetilde x$ after the introduction of the corruption. Note that $\widetilde x$ is the string 
	``edian\textquotesingle s ro'' reversed.}
	\label{fig:a:expt:shake:noise}
\end{figure*}

\myparagraph{Gaussian corruption}
This is an example of update poisoning.
	The data is not modified here but the update of a client device is directly replaced by a Gaussian random variable, with standard deviation $\sigma$ equal to the standard deviation of the original update across its components. 
	Note that we corrupt the {\em update} to the model parameters 
	transmitted by the device,
	which is typically much smaller in norm than the model parameters themselves.

\myparagraph{Omniscient corruption}
	This is an example of update poisoning.
	The data is not modified here but the parameters of a device are directly modified.
	In particular, $w_i\pow{t+1}$ for $i \in \Ccal$ is set to be
	\begin{align*}
		w_i\pow{t+1} = 
		-\frac{1}{\sum_{j \in S_t \cap \mathcal{C}} \alpha_j } 
		\Big( & 2 \sum_{j \in S_t \setminus \mathcal{C}} \alpha_j w_{j,\tau}\pow{t}
		+ 
			\sum_{j \in S_t \cap \mathcal{C}} \alpha_j w_{j,\tau}\pow{t}
			\Big)\,,
	\end{align*}
	such that
	\[
		\sum_{i \in S_t} \alpha_i w_i\pow{t+1}
			=  -\sum_{i \in S_t} \alpha_i w_{i, \tau}\pow{t} \,.
	\]
	In other words, the weighted arithmetic mean of the model parameter is set to be the negative of what it 
	would have other been without the corruption. This corruption model requires
	full knowledge of the data and server state, and is adversarial in nature.

\myparagraph{Implementation details}
Given a corruption level $\rho$, the set of devices which return corrupted updates are selected as follows:
\begin{itemize}[nolistsep,leftmargin=1em]
	\item Start with $\mathcal{C} = \varnothing$.
	\item Sample device $i$ uniformly without replacement and add to $\mathcal{C}$. 
		Stop when $\sum_{i \in C} \alpha_i$ just exceeds $\rho$.
\end{itemize}

\subsubsection{Methods}
We compare the following algorithms:
\begin{itemize}[nolistsep,leftmargin=1em]
	\item the \fedavg algorithm~\citep{mcmahan2017communication},
	\item the \rbfedavg algorithm proposed here in \Cref{algo:rfa:main},
	\item the minibatch stochastic gradient descent (SGD) algorithm. 
\end{itemize}

\subsubsection{Hyperparameters} 
The hyperparameters for each of these algorithms are detailed below.

\myparagraph{\fedavg}
The \fedavg algorithm requires the following hyperparameters.
\begin{itemize}[nolistsep,leftmargin=1em]
	\item Devices per round $m$: We use $100$ for EMNIST and $50$ for both the Shakespeare and Sent140 datasets.
	\item Batch Size and Number of Local Epochs: Instead of running $\tau$ local updates, we run for $n_e$ local epochs following \cite{mcmahan2017communication} with a batch size of $b$. For the EMNIST dataset, we use 
		$b = 50, n_e = 5$, and for Shakespeare and Sent140, we use $b=10, n_e=1$.
	\item Learning Rate $(\gamma_t)$: We use a learning a learning rate scheme 
		$\gamma_t = \gamma_0 C^{\lfloor t / t_0 \rfloor}$, where $\gamma_0$ and $C$ were tuned 
		using grid search on validation set (20\% held out from the training set) for a fixed time horizon
		on the uncorrupted data. The values which gave the highest validation accuracy were used
		{\em for all settings} - both corrupted and uncorrupted.
		The time horizon used was 2000 iterations for the EMNIST linear model, 1000 iterations for the EMNIST ConvNet
		200 iterations for Shakespeare LSTM.
	\item Initial Iterate $w\pow{0}$: Each element of $w\pow{0}$ is initialized to a uniform random variable whose 
		range is determined according to TensorFlow's ``glorot\_uniform\_initializer''.
\end{itemize}

\myparagraph{\rbfedavg}
\rbfedavg's hyperparameters, in addition to those of \fedavg, are:
\begin{itemize}[nolistsep,leftmargin=1em]
	\item Algorithm: We use the smoothed Weiszfeld algorithm, as discussed in Sec.~\ref{sec:aggr}.
	\item Smoothing parameter $\nu$: Based on the interpretation that $\nu$
		guards against division by small numbers, %
		we simply use $\nu = 10^{-6}$ throughout.
	\item Robust Aggregation Stopping Criterion: The concerns the stopping criterion used to 
		terminate the smoothed Weiszfeld algorithm. We use two criteria:
		an iteration budget and a relative improvement condition - we terminate if a given iteration budget has been extinguished, 
		or if the relative improvement in objective value $|g_\nu(v\pow{r}) - g_\nu(v\pow{r+1})| / g_\nu(v\pow{r}) \le 10^{-6}$
		is small.
\end{itemize}

\subsection{Evaluation Methodology and Other Details} \label{sec:a:expt:eval}
We specify here the quantities appearing on the $x$ and $y$ axes on the plots, 
as well as other details.

\myparagraph{$x$ Axis}
As mentioned in Section~\ref{sec:setup}, the goal of federated learning is to learn the model 
with as few rounds of communication as possible. Therefore, we evaluate various methods 
against the number of rounds of communication, which we measure via the number of calls to a secure average oracle. 

Note that \fedavg and SGD require one call to the secure average oracle per outer iteration, 
while \rbfedavg could require several. Hence, we also evaluate performance against the number of 
outer iterations.

\myparagraph{$y$ Axis}
We are primarily interested in the test accuracy, which measures the performance on unseen data.
We also plot the function value $F$, which is the quantity 
our optimization algorithm aims to minimize. We call this the train loss.

\myparagraph{Evaluation with Data Corruption}
In simulations with data corruption, while the training is performed on corrupted data, 
we evaluate train and test progress using the corruption-free data.

\myparagraph{Software}
We use the package LEAF \citep{caldas2018leaf} to simulate the federated learning setting. 
The models used are implemented in TensorFlow.%

\myparagraph{Hardware}
Each simulation was run in a simulation as a single process. The EMNIST linear model simulations
were run on two workstations with 126GB of memory, with one equipped with Intel i9 processor running at 2.80GHz, 
and the other with Intel Xeon processors running at 2.40GHz.
Simulations involving neural networks were run either on a 1080Ti or a Titan Xp GPU.

\myparagraph{Random runs}
Each simulation is repeated 5 times with different random seeds, and the solid lines in the 
plots here represents the mean over these runs, while the shaded areas show the maximum 
and minimum values obtained in these runs.

\subsection{Simulation Results: Convergence of The Smoothed Weiszfeld Algorithm}
\label{sec:a:expt:gm_algos}

For each of these models, we freeze \fedavg at a certain iteration 
and experiment with different robust aggregation algorithms. 

We find that 
the smoothed Weiszfeld algorithm enjoys a fast convergence behavior, converging exactly to the 
smoothed geometric median in a few passes. In fact, the smoothed Weiszfeld algorithm
displays (local) linear convergence, as evidenced by the straight line in log scale.
Further, we also maintain a strict iteration budget of 3 iterations. 
This choice is also justified in hindsight by the results of \Cref{fig:expt:a:hyper-niter:feml}.

Next, we visualize the weights assigned by the geometric median to the corrupted updates.
Note that the smoothed geometric median $w_1, \cdots, w_m$ 
is some convex combination $\sum_{i=1}^m \beta_i w_i$. This weight $\beta_i$ of $w_i$ 
is a measure of the influence of $w_i$ on the aggregate. We plot in \Cref{fig:expt:a:plot:gm_viz}
the ratio $\beta_i / \alpha_i$ for each device $i$, where $\alpha_i$ is its weight in the arithmetic mean
and $\beta_i$ is obtained by running the smoothed Weiszfeld algorithm to convergence.
We expect this ratio to be smaller for worse corruptions 
and ideally zero for obvious corruptions. We find that 
the smoothed geometric median does indeed assign lower weights to the corruptions, 
while only accessing 
the points via a secure average oracle.

\begin{figure*}[t!]
    \centering
    \begin{subfigure}[b]{\linewidth}
    \centering
        \includegraphics[width=0.24\textwidth]{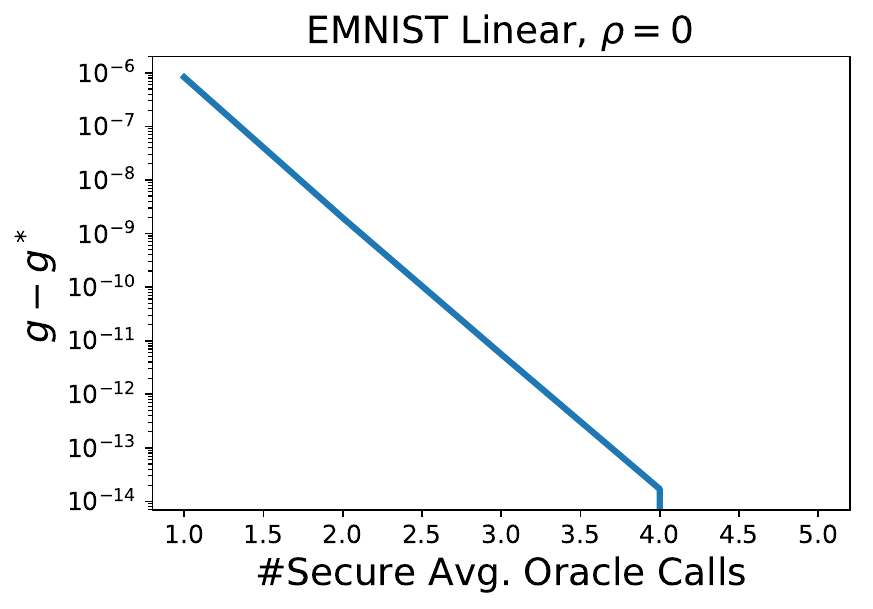}
        \includegraphics[width=0.24\textwidth]{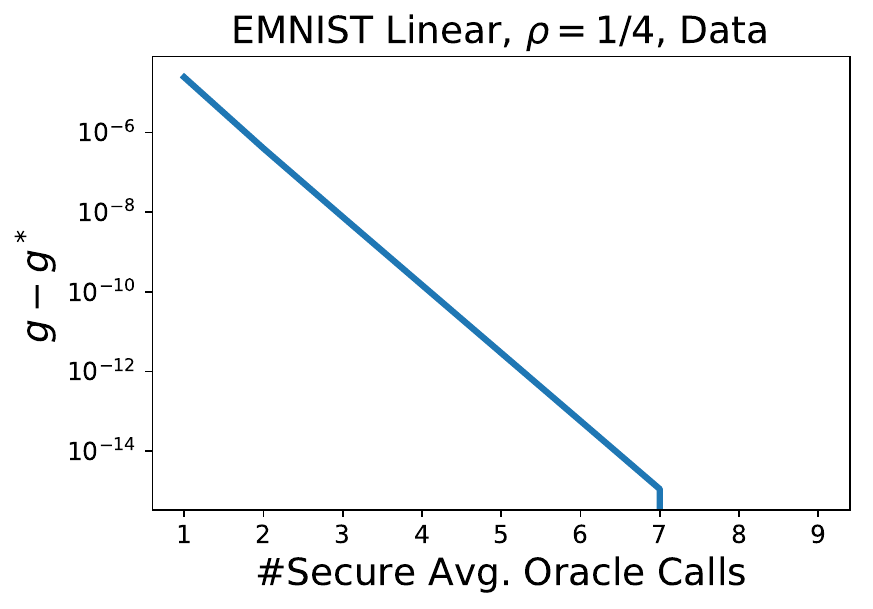}
        \includegraphics[width=0.24\textwidth]{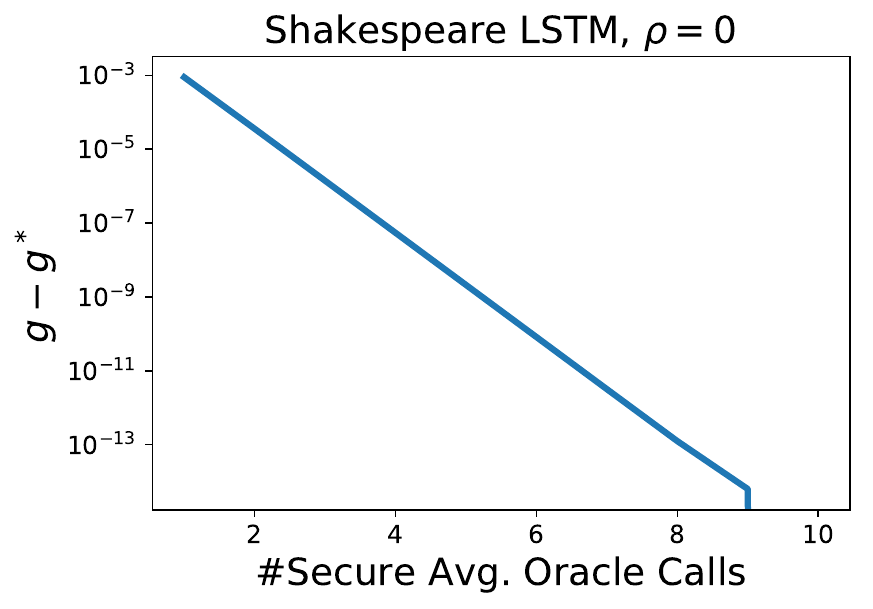}
        \includegraphics[width=0.24\textwidth]{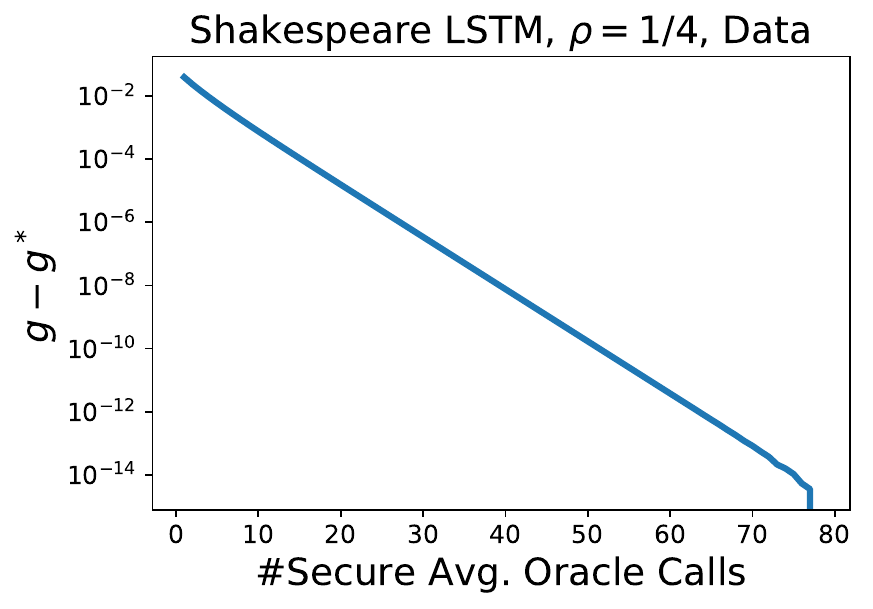}

     \caption{\small{Convergence of the smoothed Weiszfeld algorithm 
     	and for robust aggregation.}}
     	\label{fig:expt:a:plot:gm_algos}
     \end{subfigure} 

     \begin{subfigure}[b]{\linewidth}
    		\centering
        \includegraphics[width=0.19\linewidth]{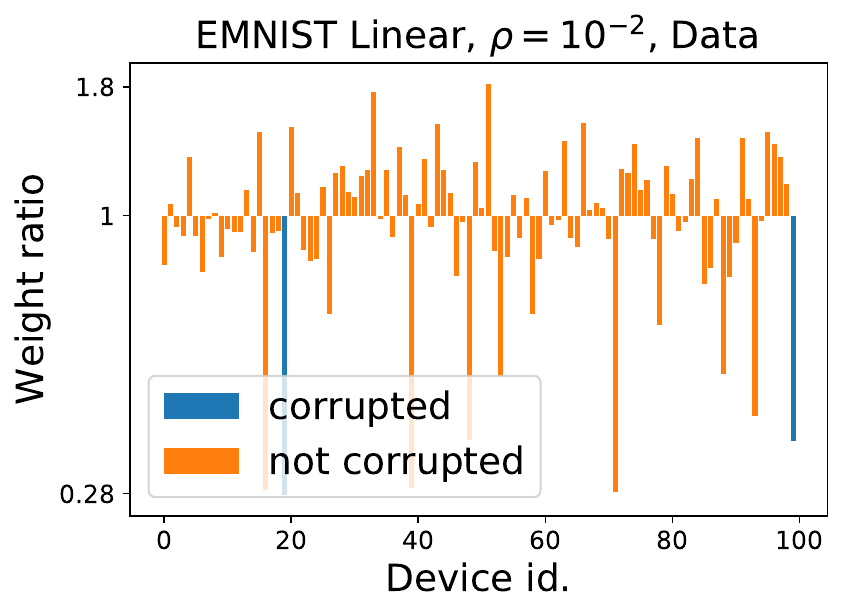}
        \includegraphics[width=0.19\linewidth]{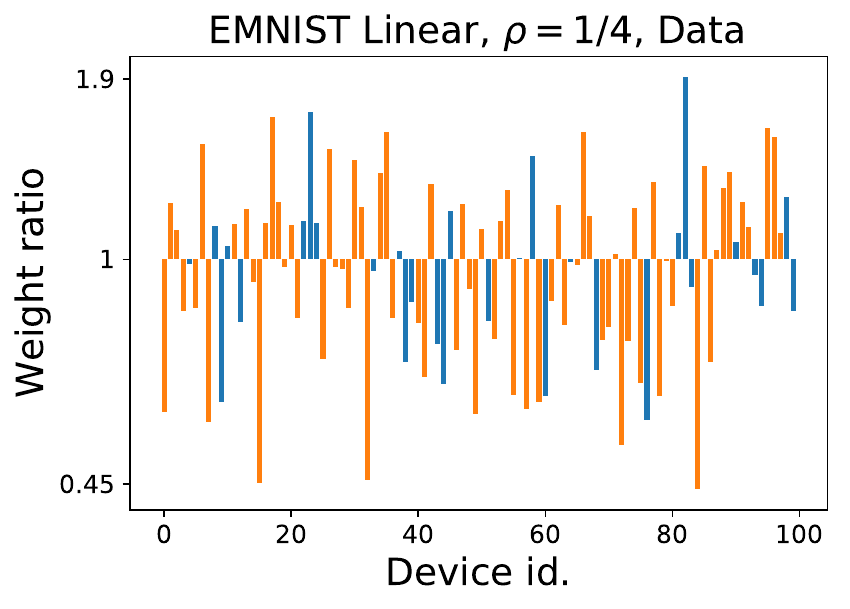}
        \includegraphics[width=0.19\linewidth]{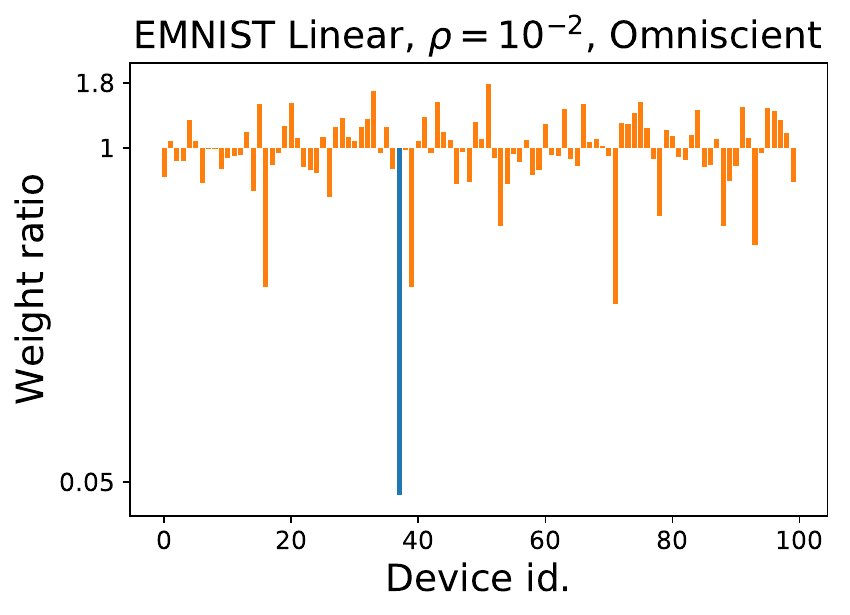}
        \includegraphics[width=0.19\linewidth]{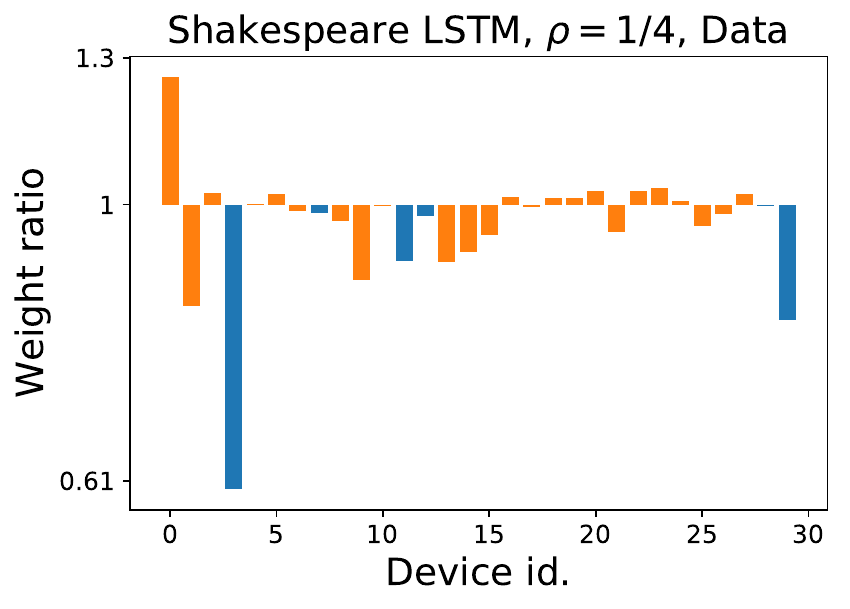}
        \includegraphics[width=0.19\linewidth]{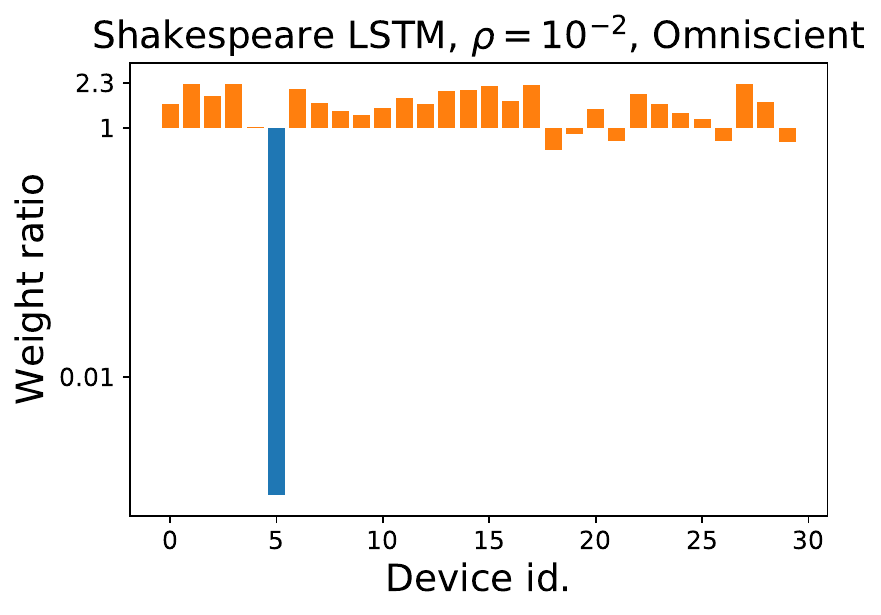}

     \caption{\small{Visualization of the re-weighting of points in the robust aggregate.}}\label{fig:expt:a:plot:gm_viz}
     \end{subfigure} 
     \caption{Performance of robust aggregation algorithms.} \label{fig:expt:a:gm_algos_main}
\end{figure*}

\subsection{Additional Simulation Results} \label{sec:a:additional-results}

\myparagraph{Effect of non-identical data distributions}
Here, we plot the analogue of \Cref{fig:expt:robustness_main} for the Sent140 dataset with data corruption in the setting where the dataset was split in an i.i.d. manner across devices. 
Recall that we had a small gap of 0.3\% between the performance of \rbfedavg and \fedavg in the setting of no corruption.
Consistent with the theory, this gap completely vanishes in the i.i.d. case, as shown in \Cref{fig:rfa:robustness-iid}.

\begin{figure}[t]
    \centering
    \includegraphics[width=0.5\linewidth]{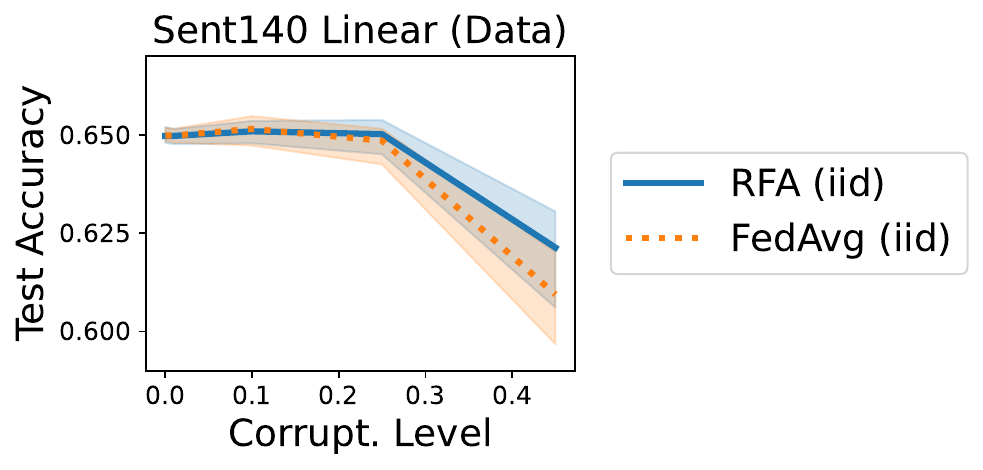}
    \caption{Robustness of \rbfedavg and \fedavg for an i.i.d. data split on Sent140 with data corruption.}
    \label{fig:rfa:robustness-iid}
\end{figure}

\myparagraph{Effect of iteration budget of smoothed Weiszfeld}
We study the effect of the iteration budget of the smoothed Weiszfeld algorithm in \rbfedavg. 
in \Cref{fig:expt:a:hyper-niter:feml}.
We observe that a low communication budget is faster in the regime of low corruption, while more iterations 
work better in the high corruption regime. We used a budget of 3 calls to the secure average oracle throughout to trade-off between
these two scenarios.

\begin{figure*}[t!]
    \centering %
        \adjincludegraphics[width=0.91\linewidth, trim={0 34pt 0 0},clip=true]{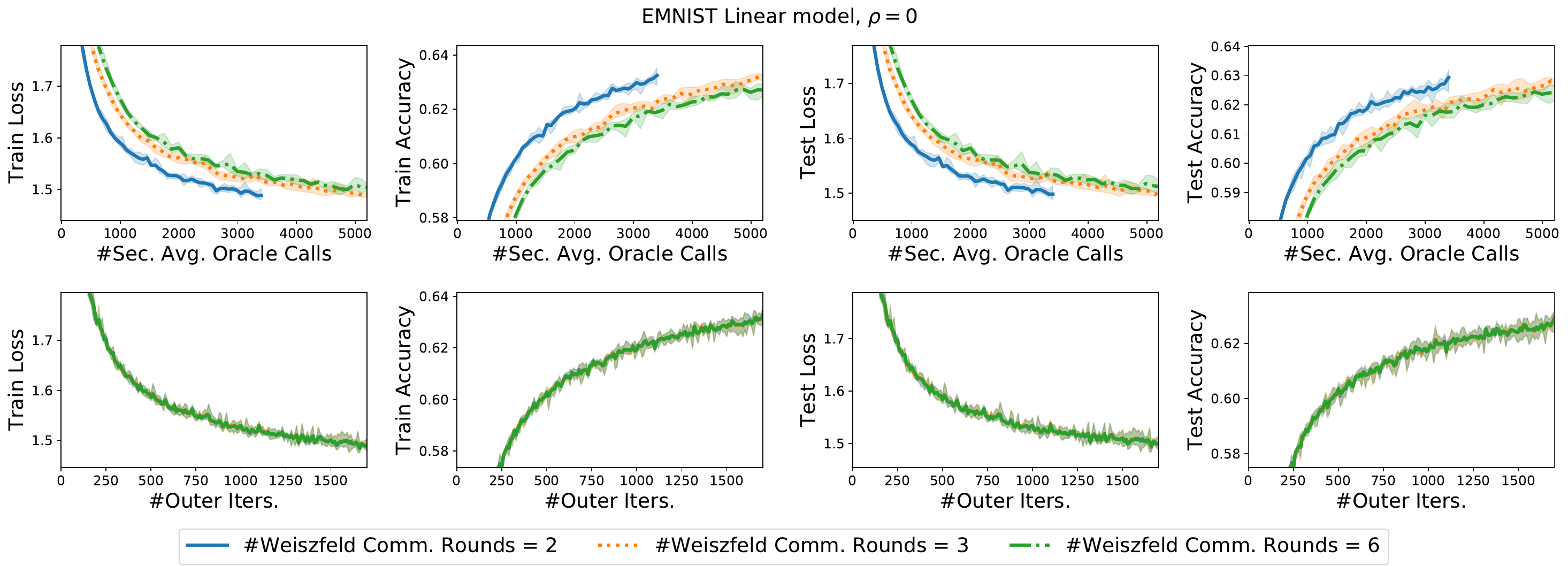}

        \adjincludegraphics[width=0.91\linewidth, trim={0 34pt 0 0},clip=true]{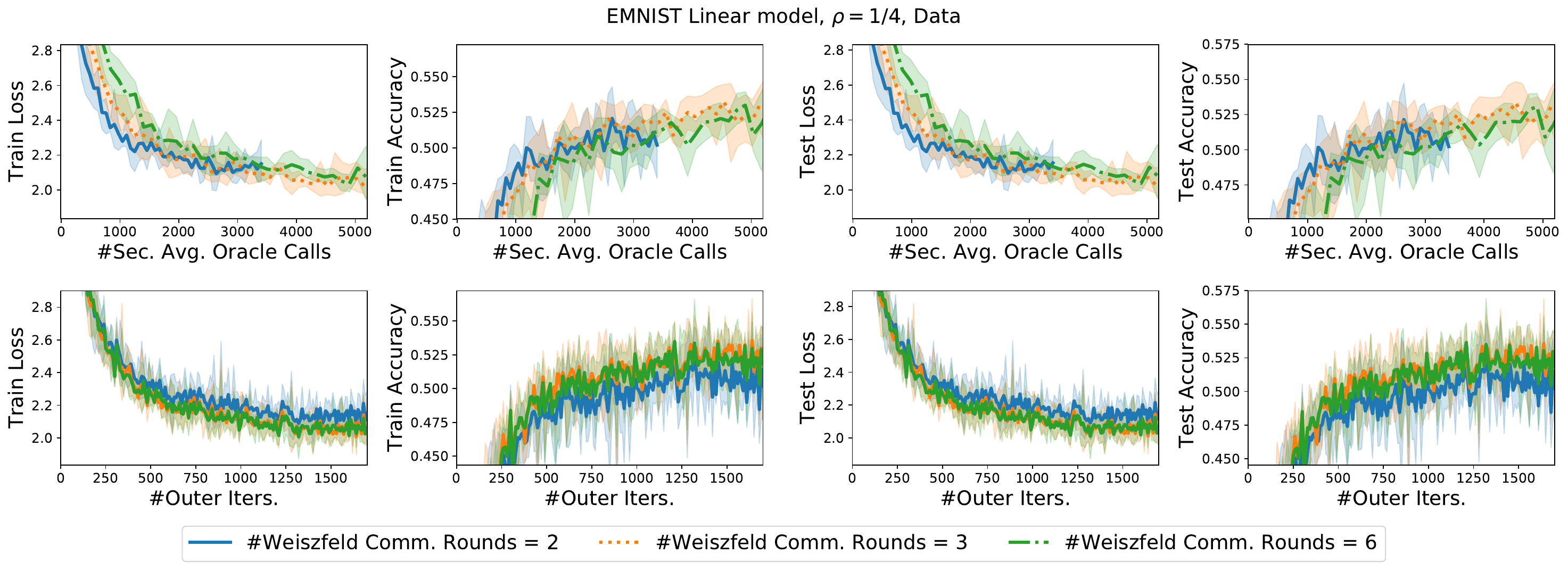}

        \adjincludegraphics[width=0.91\linewidth, trim={0 0 0 0},clip=true]{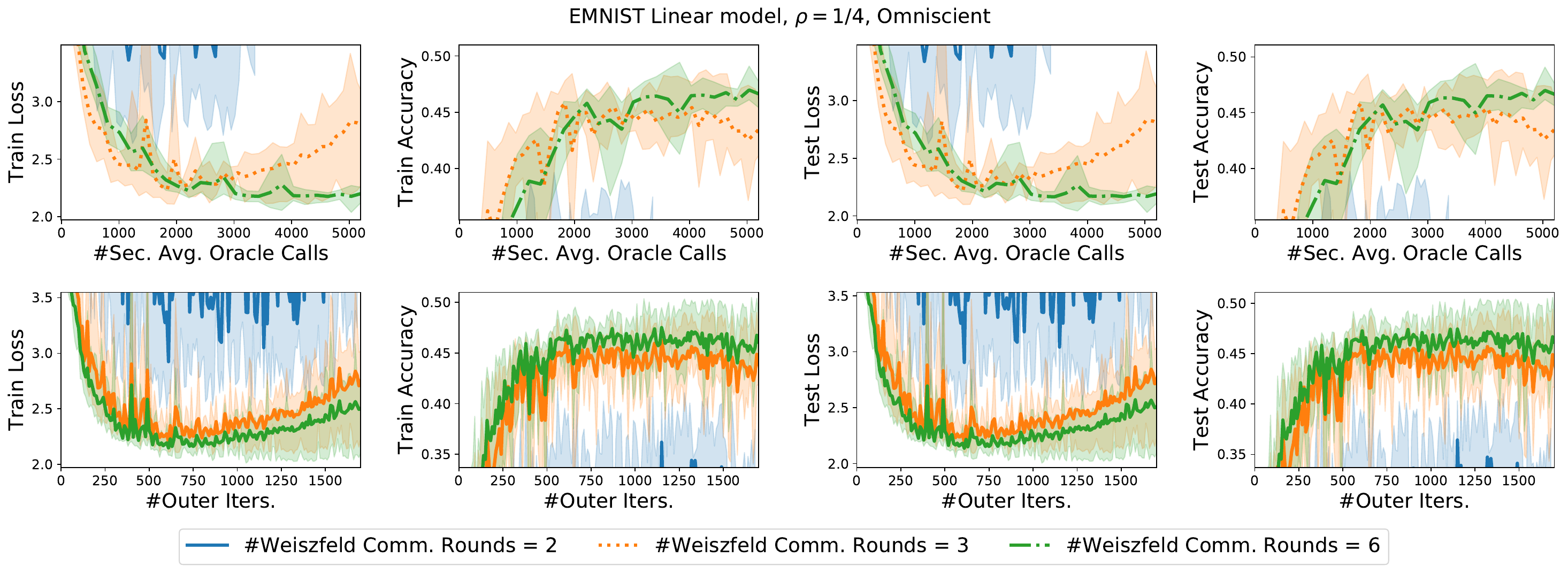}
     \caption{Hyperparameter study, effect of the maximum number of 
     	the communication budget on the smoothed Weiszfeld algorithm in \rbfedavg on the EMNIST dataset with a linear model.} 
     \label{fig:expt:a:hyper-niter:feml}
\end{figure*}

\myparagraph{Effect of number of devices per iteration round}
\Cref{fig:expt:a:hyper-ndev:feml} plots the performance of \rbfedavg against the number $m$ of devices
chosen per round. We observe the following: in the regime of low corruption, 
good performance is achieved by selecting 50 devices per round (5\%), where as 10 devices per round (1\%)
is not enough. On the other hand, in high corruption regimes, we see the benefit of choosing more devices per round, 
as a few runs with 10 or 50 devices per round with omniscient corruption at 25\% diverged.
	This is consistent with \Cref{thm:rfa:convergence}, which requires the number of devices per round to increase with the level of corruption (cf. Eq.~\eqref{eq:rfa:mainthm:ndev-per-round}). 

\begin{figure*}[t!]
    \centering %
        \adjincludegraphics[width=0.91\linewidth, trim={0 34pt 0 0},clip=true]{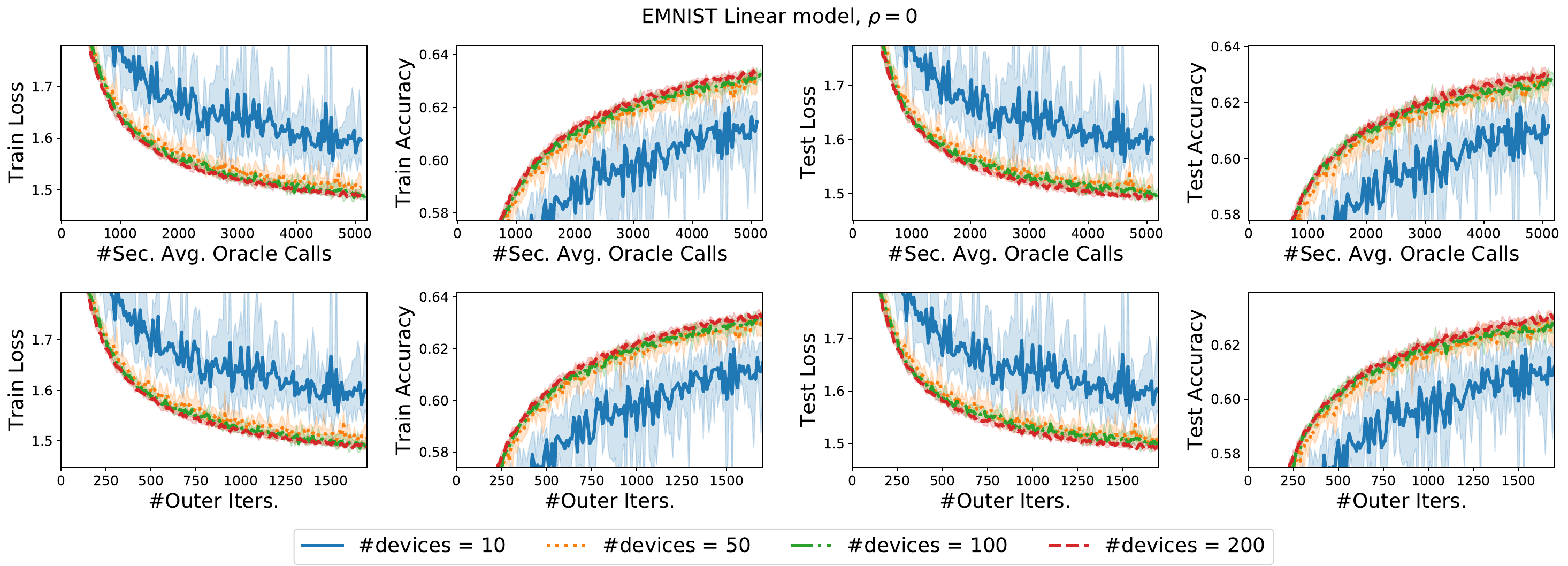}

        \adjincludegraphics[width=0.91\linewidth, trim={0 34pt 0 0},clip=true]{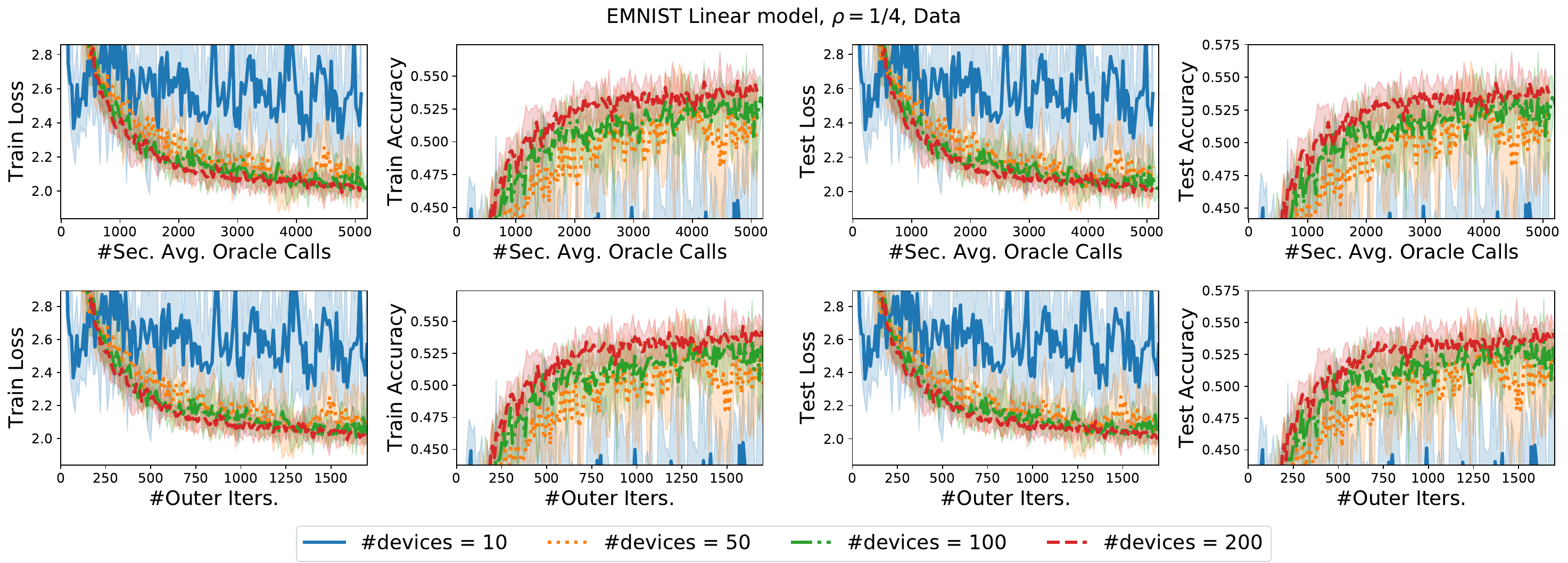}

        \adjincludegraphics[width=0.91\linewidth, trim={0 0 0 0},clip=true]{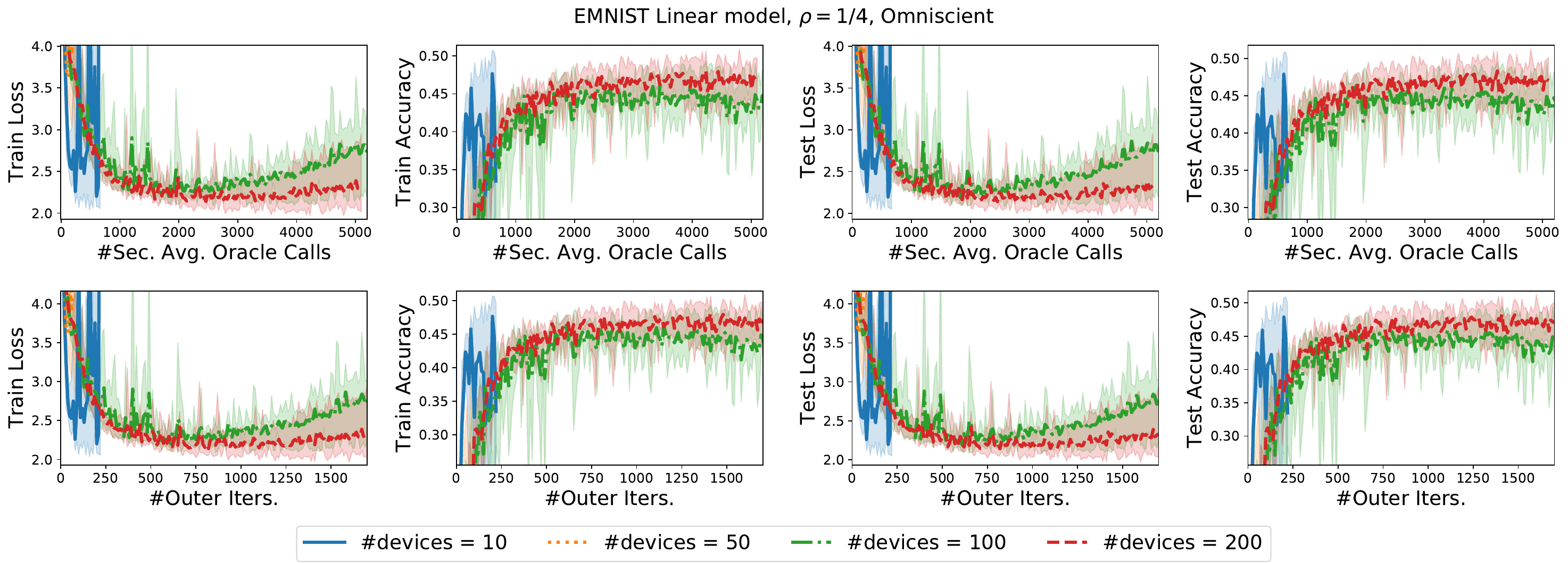}
     \caption{Hyperparameter study, effect of the number of selected client devices per round in \rbfedavg
     	 on the EMNIST dataset with a linear model.} 
     \label{fig:expt:a:hyper-ndev:feml}
\end{figure*}

\myparagraph{Effect of local computation}
\Cref{fig:expt:a:hyper-localepochs:sent} plots the performance of \fedavg and \rbfedavg versus the amount of local computation. We see that the performance is always within one standard deviation of each other irrespective of the amount of local computation. However, we also note that RFA with a single local epoch is obtains a slightly lower test accuracy in the no-corruption regime than using more local computation.

\begin{figure*}[t!]
    \centering %
     \begin{subfigure}[b]{\linewidth}
         \centering
        \adjincludegraphics[width=0.91\linewidth]{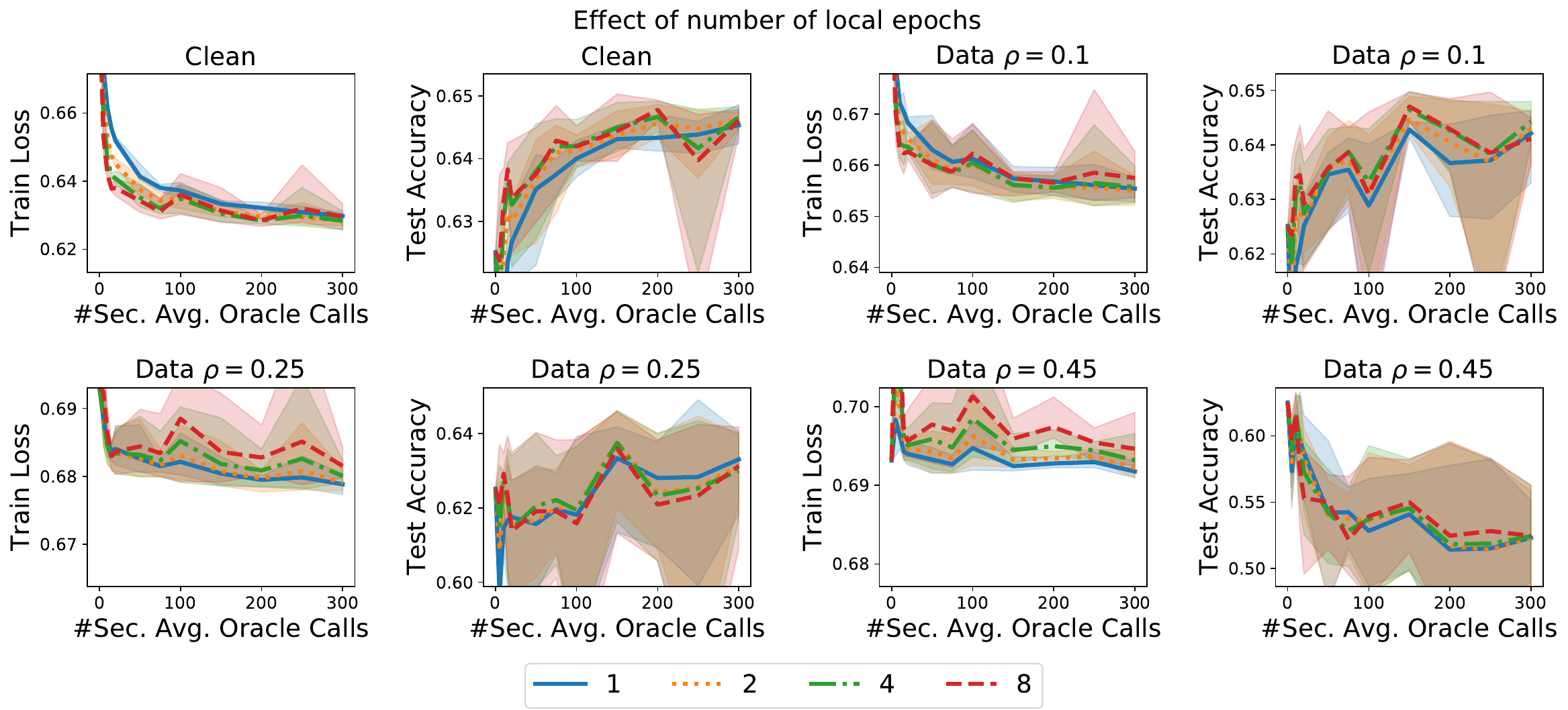}
     \caption{FedAvg.}
     \end{subfigure} 
        
      \begin{subfigure}[b]{\linewidth}
         \centering
        \adjincludegraphics[width=0.91\linewidth]{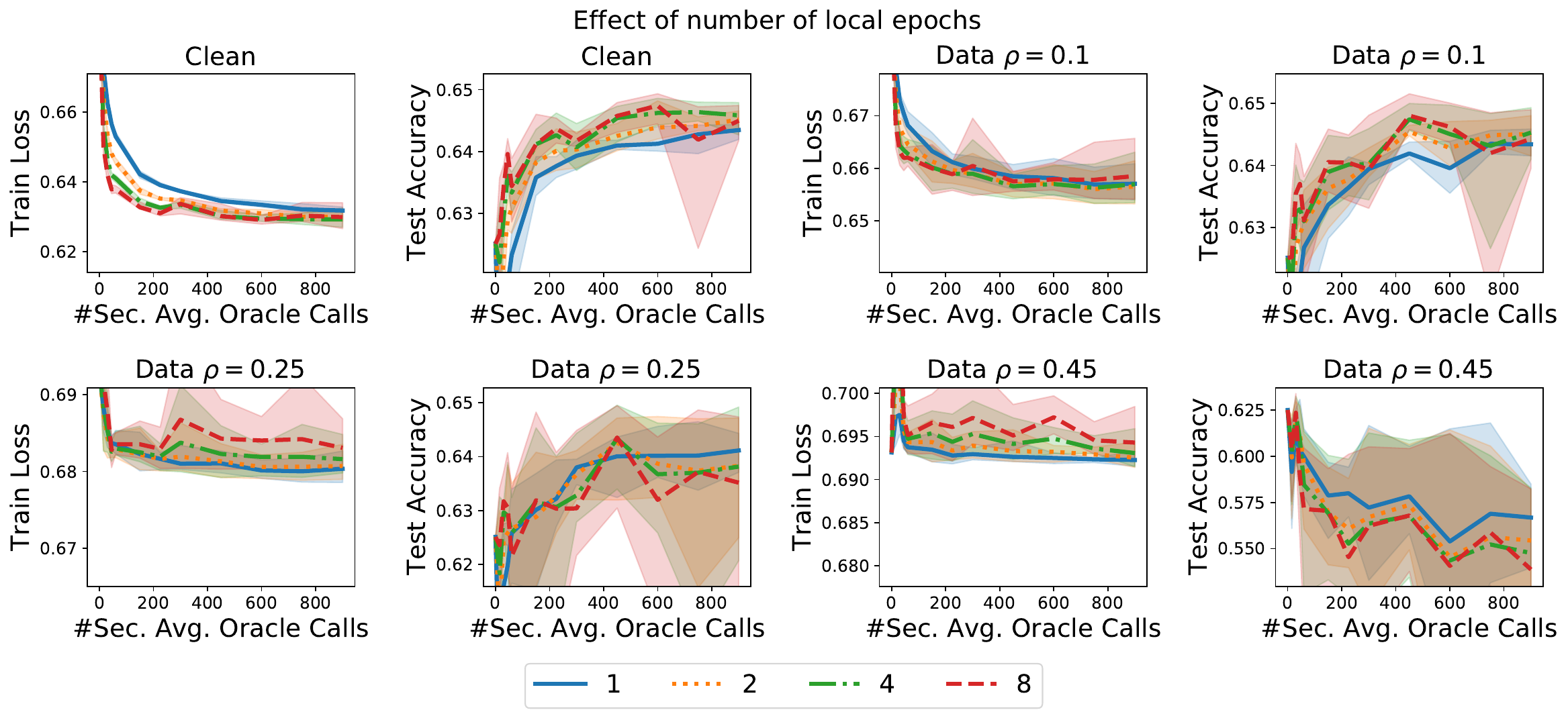}
     \caption{RFA.}
     \end{subfigure} 

     \caption{Effect of the number of epochs on FedAvg and RFA for the Sent140 dataset in the presence of data corruption.}
     \label{fig:expt:a:hyper-localepochs:sent}
\end{figure*}

\end{document}